\documentclass{article}

\usepackage[final]{neurips_2021}

\usepackage[utf8]{inputenc}
\usepackage[T1]{fontenc}
\usepackage{hyperref}
\usepackage{url}
\usepackage{booktabs}
\usepackage{amsfonts}
\usepackage{nicefrac}
\usepackage{microtype}
\usepackage{xcolor}
\usepackage{wrapfig}

\usepackage{graphicx}
\usepackage{xcolor}
\usepackage{mathtools}
\usepackage{amsmath}
\usepackage{amssymb}
\usepackage{amsthm}
\usepackage{xfrac}
\usepackage{bm}
\usepackage{enumitem}
\usepackage{algorithmic}
\usepackage[linesnumbered,lined,boxed,commentsnumbered]{algorithm2e}
\usepackage{ifthen}
\usepackage[font=small]{subfig}
\usepackage{float}

\newtheorem{assumption}{Assumption}
\newtheorem{remark}{Remark}

\newcommand{\mathsep}{,~}
\newcommand{\brute}{\textsc{MDA}}

\newcommand{\cwtm}{\textsc{TM}}

\newcommand{\cwtmfamily}{\indexvar{}{}{\overrightarrow{\cwtm}}}
\newcommand{\icwtm}{\textsc{RB-TM}}
\newcommand{\icwtmlong}{\textsc{Reliable Broadcast - Trimmed Mean}}
\newcommand{\icwtmfamily}{\indexvar{}{}{\overrightarrow{\icwtm}}}

\newcommand{\brutelong}{\textsc{Minimum--Diameter Averaging}}

\newcommand{\setr}{\mathbb{R}}
\newcommand{\setn}{\mathbb{N}}

\newcommand{\expect}{\mathop{{}\mathbb{E}}}

\newcommand{\card}[1]{\left\lvert{#1}\right\rvert}
\newcommand{\absv}[1]{\card{#1}}

\newcommand{\normtwo}[1]{\left\lVert{#1}\right\rVert_{2}}

\newcommand{\norm}[2]{\left\lVert{#1}\right\rVert_{#2}}

\newcommand{\loss}{\mathcal L}
\newcommand{\localloss}[2]{\indexvar{#1}{}{\mathcal L}\left({#2}\right)}
\newcommand{\avgloss}[1]{\bar{\mathcal L}\left( #1\right)}
\newcommand{\realgrad}[1]{\nabla{} \bar{\mathcal L}\left({#1}\right)}
\newcommand{\localgrad}[2]{\nabla{} \indexvar{#1}{}{\mathcal L}\left({#2}\right)}
\newcommand{\localgradfamily}{\overrightarrow{\nabla{} \loss}}
\newcommand{\avglocalgrad}[1]{\nabla{} \bar{\mathcal L}\left({#1}\right)}

\newcommand{\indexvar}[3]{{#3}^{\ifthenelse{\equal{#1}{}}{}{\left({#1}\right)}}_{#2}}

\newcommand{\family}[3]{\indexvar{#1}{#2}{\vec{#3}}}

\newcommand{\lipschitz}{L}
\newcommand{\bgd}{K}
\newcommand{\param}{\theta}
\newcommand{\params}[2]{\indexvar{#1}{#2}{\theta}}
\newcommand{\avgparam}[1]{\indexvar{}{#1}{\bar \theta}}

\newcommand{\byz}{\textsc{Byz}}
\newcommand{\byzfamily}[2]{\indexvar{#1}{#2}{\overrightarrow{\byz}}}

\newcommand{\grad}{g}

\newcommand{\avggrad}[1]{\indexvar{}{#1}{\bar g}}
\newcommand{\grads}[2]{\indexvar{#1}{#2}{g}}
\newcommand{\cwtmgrad}{\gamma}
\newcommand{\cwtmgrads}[2]{\indexvar{#1}{#2}{\gamma}}
\newcommand{\avgcwtmgrad}[1]{\indexvar{}{#1}{\bar \gamma}}

\newcommand{\avgeffectgrad}[1]{\indexvar{}{#1}{\bar G}}
\newcommand{\effectgrads}[2]{\indexvar{#1}{#2}{G}}
\newcommand{\noise}{\xi}
\newcommand{\noises}[2]{\indexvar{#1}{#2}{\noise}}

\newcommand{\quorum}{Q}
\newcommand{\quorums}[2]{\indexvar{#1}{#2}{\quorum}}
\newcommand{\aggr}{\textsc{Avg}}
\newcommand{\aggrfamily}[2]{\indexvar{#1}{#2}{\overrightarrow{\aggr}}}

\newcommand{\learn}{\textsc{Learn}}
\newcommand{\hlearn}{\textsc{Hom-Learn}}

\newcommand{\aggrparameter}{N}

\DeclareMathOperator*{\argmin}{arg\,min}

\renewcommand{\paragraph}[1]{\textbf{#1}~}

\newtheorem{definition}{Definition}
\newtheorem{proposition}{Proposition}
\newtheorem{theorem}{Theorem}
\newtheorem{lemma}{Lemma}
\newtheorem{corollary}{Corollary}

\renewcommand{\paragraph}[1]{\textbf{#1}~}

\title{Collaborative Learning in the Jungle\\ {(Decentralized, Byzantine, Heterogeneous, Asynchronous and Nonconvex Learning)}}

\author{
  El-Mahdi El-Mhamdi \thanks{Authors are listed alphabetically.} \\
  École Polytechnique \\
  Palaiseau, France \\
  \texttt{el-mahdi.el-mhamdi@polytechnique.edu} \\
  \And Sadegh Farhadkhani $^*$ \\
  IC School, EPFL \\
  Lausanne, Switzerland \\
  \texttt{sadegh.farhadkhani@epfl.ch} \\
  \And Rachid Guerraoui $^*$ \\
  IC School, EPFL \\
  Lausanne, Switzerland \\
  \texttt{rachid.guerraoui@epfl.ch} \\
  \And Arsany Guirguis $^*$ \\
  IC School, EPFL \\
  Lausanne, Switzerland \\
  \texttt{arsany.guirguis@epfl.ch} \\
  \And Lê-Nguyên Hoang $^*$ \\
  IC School, EPFL \\
  Lausanne, Switzerland \\
  \texttt{le.hoang@epfl.ch} \\
  \And Sébastien Rouault $^*$ \\
  IC School, EPFL \\
  Lausanne, Switzerland \\
  \texttt{sebastien.rouault@epfl.ch}
}

\begin{document}

\maketitle

\begin{abstract}
We study {\it Byzantine collaborative learning}, 
where $n$ nodes seek to collectively learn from each others' local data. 
The data distribution may vary from one node to another.
No node is trusted, and $f < n$ nodes can behave arbitrarily.
We prove that collaborative learning is equivalent to a new form of agreement, which we call \emph{averaging agreement}. 
In this problem, nodes start each with an initial vector 
and seek to approximately agree on a common vector, which is close to the average of honest nodes' initial vectors. 
We present two asynchronous solutions to averaging agreement, each we prove optimal according to some dimension.
The first, based on the minimum-diameter averaging, requires $ n \geq 6f+1$, but achieves asymptotically the best-possible averaging constant up to a multiplicative constant. 
The second, based on reliable broadcast and coordinate-wise trimmed mean, achieves optimal Byzantine resilience, i.e.,  $n \geq 3f+1$. 
Each of these algorithms induces an optimal Byzantine collaborative learning protocol.
In particular, our equivalence yields new impossibility theorems on what any collaborative learning algorithm can achieve in adversarial and heterogeneous environments.
\end{abstract}

\section{Introduction}
\label{sec:intro}

The distributed nature of data, the prohibitive cost of data transfers and the privacy concerns all call for collaborative machine learning.
The idea consists for each machine to keep its data locally and to ``simply'' exchange with other machines what it learned so far.
If all machines correctly communicate and execute the algorithms assigned to them, collaborative learning is rather easy.
It can be achieved through the standard workhorse optimization algorithm: stochastic gradient descent (SGD)~\cite{rumelhart1986learning},
which can be effectively distributed through averaging~\cite{konevcny2016federated}.

But in a practical distributed setting, hardware components may crash, software can be buggy, communications can be slowed down, data can be corrupted and machines can be hacked.
Besides, large-scale machine learning systems are trained on user-generated data, which may be crafted maliciously.
For example, recommendation algorithms have such a large influence on social medias that there are huge incentives from industries and governments to fabricate data that bias the learning algorithms and increase the visibility of some contents over others~\cite{bradshaw19,neudert2019}.
In the parlance of distributed computing, ``nodes" can be  \emph{Byzantine}~\cite{lamport1982Byzantine}{\color{black}, i.e., they can behave arbitrarily maliciously, to confuse the system}.
Given that machine learning (ML) is now used in many critical applications (e.g.,\ driving, medication, content moderation), its ability to tolerate Byzantine behavior is of paramount importance.

In this paper, we precisely define and address, for the first time, the problem of \emph{collaborative learning} in a \emph{fully decentralized}, \emph{Byzantine}, \emph{heterogeneous} and \emph{asynchronous} environment with \emph{non-convex} loss functions.
We consider $n$ nodes, which may be machines or different accounts on a social media.
Each node has its own local data, drawn from data distributions that may \emph{greatly vary} across nodes.
The nodes seek to collectively learn from each other, without however exchanging their data.
\emph{None of the nodes is trusted}, and any $f < n$ nodes can be Byzantine.

\paragraph{Contributions.} We first precisely formulate the collaborative learning problem.
Then, we give our main contribution: an equivalence between collaborative learning and a
new more abstract problem we call \emph{averaging agreement}.
More precisely, we provide two reductions: from collaborative learning to averaging agreement and from averaging agreement to collaborative learning.
We prove that both reductions essentially preserve the correctness guarantees on the output.
The former reduction is the most challenging one to design and to prove correct.
First, to update nodes' models, we use averaging agreement to aggregate nodes' stochastic gradients.
Then, to avoid model drift, we regularly ``contract'' the nodes' models using averaging agreement.
To prove correctness, we bound the diameter of honest nodes' models, and we analyze the {\it effective gradient}~\cite{podc2020}.
We then carefully select a halting iteration, for which correctness can be guaranteed.

Our tight
reduction  allows to derive both impossibility results and optimal algorithms for collaborative learning by studying the ``simpler'' averaging agreement problem.
We prove lower bounds on the correctness and Byzantine resilience that any averaging agreement algorithm can achieve, which implies the same lower bounds for collaborative learning.
We then propose two optimal algorithms for averaging agreement.
Our first algorithm is asymptotically optimal with respect to correctness, up to a multiplicative constant, when nearly all nodes are honest.
Our second algorithm achieves optimal Byzantine resilience.
Each of these algorithms induces an optimal collaborative learning protocol.

While our algorithms apply in a very general setting, they
can easily be tweaked for more specific settings with additional assumptions, such as the presence of a trusted parameter server~\cite{li2014scaling}, the assumption of homogeneous (i.e. i.i.d.) local data or synchrony~(Section \ref{sec:particular} and Section~\ref{sec:conc}).

We implemented and evaluated our algorithms in a distributed environment
with 3 ResNet models~\cite{he2016deep}.
More specifically, we present their throughput overhead
when compared to a non--robust collaborative learning approach with both i.i.d. and non--i.i.d. data (i.e. we highlight the cost of heterogeneity).
Essentially, we show that our first algorithm is more lightweight with a slowdown of at most 1.7X in the i.i.d. case and almost the triple in the non--i.i.d. case.
Our second algorithm  adds slightly more than an order of magnitude overhead: here the non-i.i.d. slowdown is twice the i.i.d. one.

\label{ref:rel_work}

\noindent
\paragraph{Related work: Byzantine learning.}
Several techniques have recently been proposed for Byzantine distributed learning, where different workers collaborate through a central \emph{parameter server}~\cite{li2014scaling} to minimize the average of their loss functions~\cite{konevcny2016federated}.
In each round, the server sends its model parameters to the workers which
use their local data to compute gradients.
Krum and Multi-Krum~\cite{krum} use a distance--based scheme to eliminate Byzantine inputs and average the remaining ones.
Median-based aggregation alternatives were also considered ~\cite{xie2018generalized}.
Bulyan~\cite{bulyanPaper} uses a meta--algorithm  against a strong adversary that can fool the aforementioned aggregation rules in high--dimensional spaces.
Coding schemes were used in
Draco~\cite{chen2018draco}
and Detox~\cite{rajput2019detox}.
In \cite{alistarh2018byzantine}, quorums of workers enable to reach an information theoretical learning optimum, assuming however a strong convex loss function.
Kardam~\cite{kardam} uses filters to tolerate Byzantine workers in an asynchronous setting.
All these approaches assume a central \emph{trusted} (parameter server) machine.

The few decentralized approaches that removed this single point of failure,
restricted however the
problem to (a) homogeneous data distribution, (b) convex functions,
and/or (c) a weak (non--Byzantine) adversary.
MOZI~\cite{guo2020towards} combines a distance--based aggregation rule
with a performance--based filtering technique, assuming
that adversaries send models
with high loss values, restricting thereby the arbitrary nature of a Byzantine agent that can craft poisoned models whose losses are small only with respect to the honest nodes' incomplete loss functions.
The technique is also inapplicable to heterogeneous learning, where nodes can have a biased loss function compared to the average of all loss functions\footnote{
Besides, MOZI, focusing on convex optimization, assumes that eventually, models on honest nodes do not drift among each others, which
may not hold for Byzantine nodes could influence the honest models to drift away from each other~\cite{baruch2019little}. }.
BRIDGE~\cite{yang2019bridge} and ByRDiE~\cite{yang2019byrdie} consider \emph{gradient} descent (GD) and  \emph{coordinate} descent (CD) optimizations, respectively.
Both rely on trimmed--mean to achieve Byzantine resilience assuming
a synchronous environment and
strongly convex loss functions\footnote{Convexity greatly helps, as the average of good models will necessarily be a good model. This is no longer the case in non-convex optimization, which includes the widely used neural network framework.} with homogeneous data distribution (\emph{i.i.d.}).
In addition, none of their optimization methods is stochastic: at each step, each node is supposed to compute the gradient on its entire local data set.
ByzSGD~\cite{podc2020} starts from the classical model of several workers and one server, which is then replicated  for Byzantine resilience.
It is assumed that up to 1/3 of the server replicas and up to 1/3 of the workers can be Byzantine, which is stronger than what we assume in the present paper where nodes play both roles and tolerate any subset of 1/3 Byzantine nodes.
More importantly, ByzSGD assumes that all communication patterns between honest servers eventually hold with probability 1;
we make no such assumption here.
Additionally, our present paper is more general, considering heterogeneous data distributions, as opposed to~\cite{podc2020}.
Furthermore, heterogeneity
naturally calls for \emph{personalized} collaborative learning~\cite{FallahMO20,HanzelyHHR20,DinhTN20,FarhadkhaniGH21},
where nodes aim to learn local models, but still leverage collaborations to improve their local models.
Interestingly,  our general scheme encompasses
personalized collaborative learning.

Maybe more importantly, our reduction to averaging agreement yields new more precise bounds that improve upon all the results listed above.
These are we believe of interest, even in more centralized, homogeneous, synchronous and convex settings.
In particular, our reduction can easily be adapted to settings where parameter servers and workers with local data are different entities, as in~\cite{podc2020}.

\noindent
\paragraph{Related work: Agreement.}
A major challenge in collaborative learning is to guarantee
``some'' agreement between nodes about the appropriate parameters to consider.
Especially in non-convex settings, this is critical as, otherwise, the gradients computed by a node may be completely irrelevant for another node.
The agreement could be achieved using the traditional \emph{consensus}  abstraction~\cite{lamport1982Byzantine}.
Yet, consensus is impossible in asynchronous environments \cite{fischer1985impossibility} and when it is possible (with partial synchrony), its usage is expensive and would be prohibitive in the context of modern ML models, with a dimension $d$ in the order of  billions.
In fact, and as we show in this paper, consensus is unnecessary.

An alternative candidate abstraction  is \emph{approximate agreement}.
This is a weak form of consensus introduced in~\cite{dolev1986reaching}  where honest nodes converge to values that are close to each other, while
{\it remaining in the convex hull} of the values proposed by honest nodes.
In the one-dimensional case, optimal convergence rate has been achieved in both synchronous~\cite{fekete1986asymptotically} and asynchronous environments~\cite{Fekete87},
while optimal asynchronous Byzantine tolerance was attained by ~\cite{abraham2004optimal}.
The multi-dimensional version was addressed by~\cite{MendesHerlihy13},
requiring however $n^d$ local computations in each round,
and assuming
$n < f(d+2)$.
This is clearly impractical in the context of modern ML.

By leveraging some distributed computing techniques~\cite{rousseeuw1985multivariate,abraham2004optimal}, we prove that collaborative learning can be reduced to \emph{averaging agreement}, which is even weaker than approximate agreement.
This enables us to bring down the requirement on the number of honest nodes from  $n > f(d+2)$ to $n > 3f$, and only require linear computation time in $d$.

\paragraph{Structure.} The rest of the paper is organized as follows.
In Section~\ref{sec:model}, we precisely define the problems we aim to solve.
Section~\ref{sec:het_loss} states our main result, namely, the equivalence between collaborative learning and averaging agreement.
Section~\ref{sec:mda_alg} describes our two solutions to averaging agreement, and proves their optimality.
Section~\ref{sec:eval} reports on our empirical evaluation and highlight important takeaways.
Finally, Section~\ref{sec:conc} concludes.
The full proofs are provided in the supplementary material, as well as the optimized algorithm for homogeneous local data.

\section{Model and Problem Definitions}
\label{sec:model}

\subsection{Distributed computing assumptions}
\label{sec:model-subsection}

We consider a standard distributed computing model with
a set $[n] = \{ 1, \ldots, n\}$ of nodes, out of which $h$ are honest and $f=n-h$ are Byzantine.
For presentation simplicity, we assume that the first $h$ nodes are honest.
But crucially, no honest node knows which $h-1$ other nodes are honest.
The $f$ Byzantine nodes know each other, can collude, and subsequently know who the $h$ remaining honest nodes are.
Essentially, we assume a single adversary that controls all the Byzantine nodes.
These nodes can send arbitrary messages, and they can send different messages to different nodes.
In the terminology of distributed computing, the adversary is omniscient but not omnipotent.
Such an adversary has access to all learning and deployment information, including the learning objective, the employed algorithm, as well as
the dataset.
We consider a general asynchronous setting~\cite{BRACHA1987130}:
the adversary can delay messages to honest nodes:
no bound on communication delays or relative speeds is assumed.
{\color{black} We denote $\byz$ the algorithm adopted by the adversary.}
Yet, we assume that the adversary is not able to delay all messages indefinitely~\cite{castro1999practical}. Besides, the adversary is not able to alter the messages from the honest nodes, which can authenticate the source of a message to prevent spoofing and Sybil attacks.

Also for presentation simplicity, we assume that processes communicate in a round-based manner~\cite{BRACHA1987130}. In each round, every honest node broadcasts a message (labelled with the round number)
and waits until it successfully gathers messages from at most $q \leq h$ other nodes (labelled with the correct round number), before performing some local computation and moving to the next round.
Even though the network is asynchronous, each round is guaranteed to eventually terminate for all honest nodes, as the $h$ honest nodes' messages will all be eventually delivered.
Evidently, however, some of them may be delivered after the node receives $q$ messages (including Byzantine nodes').
Such messages will fail to be taken into account.
Our learning algorithm will then rely
on main rounds (denoted $t$ in Section~\ref{sec:het_loss}),
each of which is decomposed into sub-rounds
that run averaging agreements.

\subsection{Machine learning assumptions}
\label{sec:assumptions}

We assume each honest node $j \in [h]$ has a local data distribution $\mathcal D_j$.
The node's local loss function is derived from the parameters $\param \in \setr^d$, the model and the local data distribution, typically through $\loss^{(j)}(\param) = \expect{}_{x \sim \mathcal D_j} [\ell(\param,x)]$, where $\ell(\param,x)$ is the loss for data point $x$, which may or may not include some regularization of the parameter $\param$.
Our model is agnostic to whether the local data distribution is a uniform distribution over collected data (i.e.,  empirical risk), or whether it is a theoretical unknown distribution the node can sample from (i.e.,  statistical risk).
We make the following assumptions about this loss function.

\begin{assumption}
\label{ass:nonnegative}
The loss functions are non-negative, i.e., $\indexvar{j}{}{\loss} \geq 0$ for all honest nodes $j \in [h]$.
\end{assumption}

\begin{assumption}
\label{ass:lipschitz}
The loss functions are $\lipschitz$-smooth, i.e., there exists a constant $\lipschitz$ such that
\begin{equation}
  \forall \param, \param' \in \mathbb R^d \mathsep \forall j \in [h] \mathsep \normtwo{\localgrad{j}{\param} - \localgrad{j}{\param'}} \leq \lipschitz \normtwo{\param - \param'}.
\end{equation}
\end{assumption}

\begin{assumption}
\label{ass:variance}
The variance of the noise in the gradient estimations is uniformly bounded, i.e.,
\begin{equation}
  \forall j \in [h] \mathsep \forall \param \in \mathbb R^d \mathsep
  \expect_{x \sim \mathcal D_j} \normtwo{\nabla_{\param} \ell(\param,x) - \nabla \loss^{(j)}(\param)}^2
  \leq \sigma^2.
\end{equation}
Moreover, the data samplings done by two different nodes are independent.
\end{assumption}

\begin{assumption}
\label{ass:loss_upper_bound}
  There is a computable bound $\loss_{max}$ such that,  at initial point $\indexvar{}{1}{\param} \in \setr^d$, for any honest node $j \in [h]$, we have $\loss^{(j)}(\indexvar{}{1}{\param}) \leq \loss_{max}$.
\end{assumption}

While the first three assumptions are standard, the fourth assumption deserves further explanation. Notice first that $\theta_1$ is a given parameter of our algorithms, which we could, just for the sake of the argument, set to $0$. The assumption would thus be about the value of the local losses at $0$, which will typically depend on the nodes' local data distribution. But losses usually depend on the data only as an average of the loss per data point. Moreover, the loss at $0$ for any data point is usually bounded. In image classification tasks for example, each color intensity of each pixel of an input image has a bounded value. This usually suffices to upper-bound the loss at $0$ for any data point, which then yields Assumption~\ref{ass:loss_upper_bound}.

In iteration $t$, we require each node to average the stochastic gradient estimates over a batch of $b_t$ i.i.d. samples.
As a result, denoting $\params{j}{t}$ and $\grads{j}{t} \triangleq \frac{1}{b_t} \sum_{i \in [b_t]} \nabla_{\param} \ell(\params{j}{t},x_{t,i}^{(j)})$ node $j$'s parameters and computed stochastic gradient in iteration $t$, we have
\begin{equation}
  \expect_{x \sim \mathcal D_j} \normtwo{\grads{j}{t} - \nabla \loss^{(j)}(\params{j}{t})}^2
  \leq \frac{\sigma^2}{b_t} \triangleq \sigma_t^2.
\end{equation}
As $t$ grows, we increase batch size $b_t$ up to $\Theta(1/\delta^2)$ (where $\delta$ is a parameter of the collaborative learning problem, see Section \ref{sec:collaborative_learning_definition}),
so that $\sigma_t = \mathcal O(\delta)$ for $t$ large enough (see Remark~\ref{remark:batch_size}).
This allows to dynamically mitigate the decrease of the norm of the true gradient.
Namely, early on, while we are far from convergence, this norm is usually large.
It is then desirable to have very noisy estimates, as these can be obtained more efficiently, and as the aggregation of these poor estimates will nevertheless allow progress.
However, as we get closer to convergence, the norm of the true gradient becomes smaller,  making the learning more vulnerable to Byzantine attacks~\cite{baruch2019little}.
Increasing the batch size then becomes useful.
Our proofs essentially formalize the extent to which the batch size needs to increase.

\subsection{Collaborative learning}
\label{sec:collaborative_learning_definition}

Given the $\ell_2$ diameter  $\Delta_2(\family{}{}{\param}) = \max_{j,k \in [h]} \normtwo{\indexvar{j}{}{\param} - \indexvar{k}{}{\param}}$,
collaborative learning consists in minimizing the average $\avgloss{{\color{black} \avgparam{}}} \triangleq \frac{1}{h} \sum_{j \in [h]} \localloss{j}{\color{black} \avgparam{}}$ of local losses {\color{black} at the average $\avgparam{} \triangleq \frac{1}{h} \sum_{j \in [h]} \params{j}{}$}, while guaranteeing that the honest nodes' parameters have a small diameter.

This general model encompasses to the \emph{personalized} federated learning problem introduced by~\cite{FallahMO20,HanzelyHHR20,DinhTN20,FarhadkhaniGH21}.
For instance, in~\cite{HanzelyHHR20}, each node $j$ aims to learn a local model $x_j$ that minimizes $f_j$, with a penalty $\frac{\lambda}{2} \normtwo{x_j - \bar x}^2$ on their distance to the average $\bar x$ of all models.
This framework can be restated by considering that nodes must agree on a common parameter $\param = \bar x$, but have local losses defined by $\localloss{j}{\param} \triangleq \min_{x_j} f_j(x_j) + \frac{\lambda}{2} \normtwo{x_j - \param}^2 $.
The problem of~\cite{HanzelyHHR20} then boils down to minimizing the average of local losses.

\begin{definition}
\label{def:collab_learning}
An algorithm \learn{} solves the Byzantine $C$-collaborative learning problem if,
given any local losses $\loss^{(j)}$ for $j \in [h]$ satisfying
assumptions (\ref{ass:nonnegative},\ref{ass:lipschitz},\ref{ass:variance},\ref{ass:loss_upper_bound}) and
any $\delta > 0$,
{\color{black} no matter what Byzantine attack $\byz$ is adopted by Byzantines,}
\learn{} outputs a vector family $\family{}{}{\param}$ of honest nodes such that
\begin{equation}
\label{eq:constructive_learning_convergence}
    \expect \Delta_2(\family{}{}{\param})^2 \leq \delta^2 \quad \text{and} \quad
    \expect \normtwo{ \avglocalgrad{\avgparam{}} }^2 \leq (1+\delta)^2 C^2 \bgd^2,
\end{equation}
where $\bgd \triangleq \sup_{j, k \in [h], } \sup_{\param \in \mathbb R^d} \normtwo{\localgrad{j}{\param} - \localgrad{k}{\param}}$ is the largest difference between the true local gradients at the same parameter $\param$, and where the randomness comes from the algorithm (typically the random sampling for gradient estimates).
\end{definition}

In our definition above, the constant $\bgd$ measures the heterogeneity of local data distributions.
Intuitively, this also captures the hardness of the problem.
Indeed, the more heterogeneous the local data distributions, the more options Byzantine nodes have to bias the learning, the harder it is to learn in a Byzantine-resilient manner.
Interestingly, for convex quadratic losses, our guarantee implies straightforwardly an upper-bound on the distance to the unique optimum of the problem, which is proportional to the hardness of the problem measured by $\bgd$.
In fact, our equivalence result conveys the tightness of this guarantee.
In particular, the combination of our equivalence and of Theorem \ref{th:lower_bound_averaging_constant} implies that, for any $\varepsilon > 0$, asynchronous $(2f/h-\varepsilon)$-collaborative learning is impossible.

\subsection{Averaging agreement}

We address collaborative learning by reducing it to a new abstract distributed computing problem, which we call \emph{averaging agreement}.

\begin{definition}
\label{def:averaging}
    A distributed algorithm $\aggr{}$ achieves Byzantine {\it $C$-averaging agreement} if, for any input $\aggrparameter \in \mathbb{N}$,
    any vector family $\family{}{}{x} \in \mathbb R^{d \cdot h}$ and any Byzantine attack {\color{black} $\byz$},
    denoting ${\color{black} \family{}{}{y} \triangleq \aggr{}_{\aggrparameter} (\family{}{}{x}, \byz)}$ the output of \aggr{} given such inputs,
    we guarantee
    \begin{equation}
    \label{eq:avg_agreement}
        \expect \Delta_2({\color{black} \family{}{}{y}})^2 \leq \frac{\Delta_2(\family{}{}{x})^2}{4^\aggrparameter} \quad \text{and} \quad \expect \normtwo{{\color{black} \bar y} - \bar x}^2 \leq C^2 \Delta_2(\family{}{}{x})^2,
    \end{equation}
    where ${\color{black} \indexvar{}{\aggrparameter}{\bar y} = \frac{1}{h} \sum_{j \in [h]} \indexvar{j}{\aggrparameter}{y}}$ is the average of honest nodes' vectors,
    and where the randomness comes from the algorithm.
    We simply say that an algorithm solves averaging agreement if there exists a constant $C$ for which it solves $C$-averaging agreement.
\end{definition}

In particular, for deterministic algorithms, $C$-averaging {\color{black} on input $\aggrparameter$} ensures the following guarantee
\begin{equation}
\label{eq:avg_agreement_deterministic}
    \Delta_2({\color{black} \family{}{}{y}}) \leq \frac{\Delta_2(\family{}{}{x})}{2^\aggrparameter} \quad \text{and} \quad \normtwo{{\color{black} \bar{y}} - \bar x} \leq C \Delta_2(\family{}{}{x}).
\end{equation}
In Section~\ref{sec:mda_alg}, we will present two solutions to the averaging agreement problem.
These solutions typically involve several rounds.
At each round, each node sends their current vector to all other nodes.
Then, once a node has received sufficiently many vectors, it will execute a robust mean estimator to these vectors, the output of which will be their starting vector for the next round.
The nodes then halt after a number of rounds dependent on the parameter $\aggrparameter$.

\section{The Equivalence}
\label{sec:het_loss}

The main result of this paper
is that, for $\bgd>0$, $C$-collaborative learning is equivalent to $C$-averaging agreement. We present two reductions, first from collaborative learning to averaging agreement, and then from averaging agreement to collaborative learning.

\subsection{From collaborative learning to averaging agreement}
\label{subsec:reduction_het}

Given an algorithm \aggr{} that solves Byzantine $C$-averaging agreement, we design a Byzantine collaborative learning algorithm \learn{}. Recall that \learn{} must take a constant $\delta >0$ as input, which determines the degree of agreement (i.e., learning quality) that \learn{} must achieve.

All honest parameter vectors are initialized with the same random values (i.e.,\ $\forall j \in [h], \params{j}{1}=\params{}{1}$) using a pre-defined seed.
At iteration $t$, each honest node $j \in [h]$ first computes a local gradient estimate $\indexvar{j}{t}{\grad}$ given its local loss function $\localloss{j}{\cdot}$ and its local parameters $\params{j}{t}$, with a batch size $b_t$.
 But, instead of performing a learning step with this gradient estimate, \learn{} uses an aggregate of all local gradients, which we compute using the averaging agreement algorithm \aggr{}.

Recall from Definition \ref{def:averaging} that \aggr{} depends on a parameter which defines the degree of agreement.
We set this parameter at $\aggrparameter(t) \triangleq \left\lceil \log_2 t \right\rceil$ at iteration $t$, so that $1/4^{\aggrparameter(t)} \leq 1/t^2$.
Denoting $\family{}{t}{\cwtmgrad}$ the output of $\aggrfamily{}{\aggrparameter(t)}$ applied to vectors $\family{}{t}{\grad}$, we then have the following guarantee:
\begin{equation}
\label{eq:cwtmgrad_guarantee}
    \expect \Delta_2\left( \family{}{t}{\cwtmgrad} \right)^2 \leq \frac{\Delta_2\left( \family{}{t}{\grad} \right)^2}{t^2}
    \quad \text{and} \quad
    \expect \normtwo{ \avgcwtmgrad{t} - \avggrad{t} }^2
    \leq C^2 \Delta_2\left( \family{}{t}{\grad} \right)^2,
\end{equation}
where the expectations are conditioned on $\family{}{t}{\grad}$.
We then update node $j$'s parameters by $\params{j}{t+1/2} = \params{j}{t} - \eta \cwtmgrads{j}{t}$,
for a fixed learning rate $\eta \triangleq \delta/12\lipschitz$.
But before moving on to the next iteration, we run once again \aggr{}, with its parameter set to 1. Moreover, this time, \aggr{} is run on local nodes' parameters.
Denoting $\family{}{t+1}{\param}$ the output of $\aggrfamily{}{1}$ executed with vectors $\family{}{t+1/2}{\param}$, we then have
\begin{equation}
\label{eq:param_guarantee}
    \expect \Delta_2\left( \family{}{t+1}{\param} \right)^2 \leq \frac{\Delta_2 ( \family{}{t+1/2}{\param} )^2}{4}
    \quad \text{and} \quad
    \expect \normtwo{ \avgparam{t+1} - \avgparam{t+1/2}}^2
    \leq C^2 \Delta_2( \family{}{t+1/2}{\param} )^2,
\end{equation}
where the expectations are conditioned on $\family{}{t+1/2}{\param}$.
On input $\delta$, \learn{} then runs $T \triangleq T_{\learn{}}(\delta)$ learning iterations.
The function $T_{\learn{}} (\delta)$ will be given explicitly in the proof of Theorem \ref{th:gradient_near_convergence}, where we will stress the fact that it can be computed from the inputs of the problem\footnote{Note that upper-bounds (even conservative) on such values suffice. Our guarantees still hold, though $T_{\learn{}} (\delta)$ would take larger values, which makes \learn{} slower to converge to $\family{}{*}{\param}$.} ($\lipschitz$, $\bgd$, $C$, $n$, $f$, $\sigma_t$, $\loss_{max}$ and $\delta$).
Finally, instead of returning $\family{}{T_{\learn{}}(\delta)}{\param}$, \learn{} chooses uniformly randomly an iteration $* \in [T_{\learn{}}(\delta)]$, using a predefined common seed, and returns the vector family $\family{}{*}{\param}$.
We recapitulate the local execution of \learn{} (at each node) in Algorithm~\ref{alg:learn}.
{\color{black} We stress that on steps 6 and 8, when the averaging agreement algorithm \aggr{} is called, the Byzantines can adopt any procedure $\byz$, which consists in sending any message to any node at any point based on any information in the system, and in delaying for any amount of time any message sent by any honest node. Note that, apart from this, all other steps of Algorithm~\ref{alg:learn} is a purely local operation.}

\begin{algorithm}[H]
\label{alg:learn}
 \caption{\learn{} execution on a honest node.}
    \SetKwFunction{GradientOracle}{GradientOracle}
    \SetKwFunction{QueryGradients}{QueryGradients}
    \SetKwFunction{QueryParameters}{QueryParameters}
    \SetKwFunction{Broadcast}{Broadcast}
    \KwData{Local loss gradient oracle, parameter $\delta > 0$}
    \KwResult{Model parameters $\params{}{}$}
    \BlankLine
    \BlankLine
    Initialize local parameters $\params{}{1}$ using a fixed common seed\;
    Fix learning rate $\eta \triangleq \delta/12\lipschitz$\;
    Fix number of iterations $T \triangleq T_\learn{}(\delta)$\;
    \For{$t \leftarrow 1, \ldots, T$}{
        $\grads{}{t} \leftarrow \GradientOracle(\params{}{t}, b_t)$\;
        $\cwtmgrads{}{t} \leftarrow \aggr{}_{\aggrparameter(t)} {\color{black} (\family{}{t}{\grad}, \byz)}$ \tcp{Vulnerable to Byzantine attacks}
        $\params{}{t+1/2} \leftarrow \params{}{t} - \eta \cwtmgrads{}{t}$\;
        $\params{}{t+1} \leftarrow \aggr{}_1 {\color{black} \left(\family{}{t+1/2}{\param}, \byz \right)}$  \tcp{Vulnerable to Byzantine attacks}
    }
    Draw $* \sim \mathcal U([T])$ using the fixed common seed\;
    Return $\param_{*}$\;
\end{algorithm}

\begin{remark}
In practice, it may be more efficient to return the last computed vector family, though our proof applies to a randomly selected iteration.
\end{remark}

\begin{theorem}
\label{th:gradient_near_convergence}
Under assumptions (\ref{ass:nonnegative}, \ref{ass:lipschitz}, \ref{ass:variance}, \ref{ass:loss_upper_bound}) and $\bgd >0$,
given a $C$-averaging agreement oracle \aggr{}, on any input $0 <\delta < 3$,
\learn{} solves Byzantine $C$-collaborative learning.
\end{theorem}

The proof is quite technical, and is provided in the supplementary material.
Essentially, it focuses on the average of all honest nodes' local parameters, and on the \emph{effective gradients} that they undergo, given the local gradient updates and the applications of the averaging agreement oracle \aggr{}.

\begin{remark}
\label{remark:batch_size}
Our proof requires
$T_{\learn{}}(\delta) = \Theta\Big( \delta^{-1} \max \big\{ \delta^{-2}, (t \mapsto \sup_{s \geq t} \sigma_s)^{-1} (\delta) \big\} \Big)$.
To prevent the noise from being the bottleneck for the convergence rate, we then need $\sigma_{t} = \Theta(\delta)$, so that $T_{\learn{}} = \Theta(\delta^{-3})$.
Interestingly, this can be obtained by, for example, setting $b_t \triangleq t$, for $t \leq T_1 = \Theta(\delta^{-2})$, and $b_t \triangleq T_1$ for $t > T_1$, where $T_1$ is precisely defined by the proof provided in the supplementary material.
In particular, we do not need to assume that $b_t \rightarrow \infty$.
\end{remark}

\subsection{Converse reduction}

We also prove the converse reduction, from averaging agreement to collaborative learning.

\begin{theorem}
\label{th:inverse_redcution}
Given a Byzantine $C$-collaborative learning oracle, then, for any $\delta > 0$, there is a solution to Byzantine $(1+\delta)C$-averaging agreement.
\end{theorem}

The proof is given in the supplementary material.
It is obtained by applying a collaborative learning algorithm \learn{} to the local loss functions $\loss^{(j)} (\param) \triangleq \normtwo{\param - x^{(j)}}^2$.

This converse reduction proves the tightness of our former reduction, and allows to straightforwardly derive impossibility theorems about collaborative learning from impossibility theorems about averaging agreement.
In particular, in the sequel, we prove that no asynchronous algorithm can achieve better than Byzantine $\frac{h+2f-q}{h}$-averaging agreement.
It follows that no algorithm can achieve better than Byzantine $\frac{h+2f-q}{h}$-collaborative learning.
Similarly, Theorem \ref{th:cwtm}
implies the impossibility of Byzantine asynchronous collaborative learning for $n \leq 3f$.

\subsection{Particular cases}
\label{sec:particular}

\paragraph{Trusted server.}
It is straightforward to adapt our techniques and prove the equivalence between  averaging agreement and collaborative learning in a context
with a trusted server.
Our lower bounds for asynchronous collaborative learning still apply to $C$-averaging agreement, and thus to $C$-collaborative learning.
Note, however, that the trusted server may allow to improve the speed of collaborative learning, as it no longer requires contracting the parameters of local nodes' models.

\paragraph{Homogeneous learning.} In the supplementary material, we propose a faster algorithm for i.i.d. data, called \hlearn{}, which skips the averaging agreement of nodes' gradients.
Despite requiring fewer communications, \hlearn{} remains correct, as the following theorem shows.

\begin{theorem}
\label{th:homo}
Under assumptions (\ref{ass:nonnegative}, \ref{ass:lipschitz}, \ref{ass:variance}, \ref{ass:loss_upper_bound}), for i.i.d. local data and given a C-averaging agreement oracle \aggr{}, on input $\delta>0$,
\hlearn{} guarantees
$\expect\Delta_2 (\family{}{*}{\param})^2 \leq \delta^2 $ and $\expect \normtwo{\localgrad{}{\avgparam{*}}}^2 \leq \delta^2$.
\end{theorem}

\section{Solutions to Averaging Agreement}
\label{sec:mda_alg}

We now present two solutions to the averaging agreement problem, called \brutelong{}\footnote{Introduced by~\cite{bulyanPaper} in the context of robust machine learning, it uses the same principle as the minimal volume ellipsoid, that was introduced by~\cite{rousseeuw1985multivariate} in the context of robust statistics.} (\brute{}) and \icwtmlong{} (\icwtm{}), each thus inducing a solution to collaborative learning. We prove each optimal according to some dimension.

\subsection{Optimal averaging}

Given a family $\family{}{}{z} \in \mathbb R^{d \cdot q}$ of vectors, \brute{} first identifies a subfamily $S_{\brute{}}(\family{}{}{z})$ of $q-f$ vectors of minimal $\ell_2$ diameter, i.e.,
\begin{equation}
  S_{\brute{}} (\family{}{}{z}) \in \argmin_{\underset{\card{S} = q-f}{S \subset [q]}} \Delta_2 \left(\family{S}{}{z}\right) = \argmin_{\underset{\card{S} = q-f}{S \subset [q]}} \max_{j,k \in S} \normtwo{\indexvar{j}{}{z} - \indexvar{k}{}{z}}.
\end{equation}
We denote $\family{\brute{}}{}{z} \in \mathbb R^{d \cdot (q-f)}$ the subfamily thereby selected. \brute{} then outputs the average of this subfamily, i.e.,
\begin{equation}
  \brute{} (\family{}{}{z}) \triangleq \frac{1}{q-f} \sum_{j \in S_{\brute{}} (\family{}{}{z})} \indexvar{j}{}{z}.
\end{equation}
On input $\aggrparameter \in \setn$, $\brute{}_N$ iterates \brute{} $T_\brute{}(\aggrparameter) = \left\lceil \aggrparameter \ln 2/ \tilde \varepsilon\right \rceil$ times on vectors received from other nodes at each communication round such that the output of round $t$ will be the input of round $t+1$.
The correctness of \brute{} is then ensured under the following assumption.

\begin{assumption}[Assumption for analysis of \brute{}]
\label{ass:query_brute}
  There is $0 < \varepsilon < 1$ such that $n \geq \frac{6+2\varepsilon}{1-\varepsilon} f$. This then allows to set $q \geq \frac{1+\varepsilon}{2} h + \frac{5+3\varepsilon}{2} f$. In this case, we define $\tilde \varepsilon \triangleq \frac{2 \varepsilon}{1+\varepsilon}$.
\end{assumption}

\begin{theorem}
\label{th:brute}
   Under Assumption \ref{ass:query_brute},
 \brute{} achieves Byzantine $\frac{(2f+h-q)q + (q-2f)f}{h(q-f) \tilde \varepsilon}$-averaging agreement.
\end{theorem}

The proof is given in the supplementary material.
It relies on the observation that, because of the filter, no Byzantine vector can significantly harm the estimation of the average.

\begin{remark}
Although \brute{} runs in linear time in $d$, it runs in exponential time in $q$.
Interestingly, assuming that each honest node fully trusts its computations and its data (which may not hold if parameter-servers do not compute gradients
as in \cite{li2014scaling}),
each honest node can use its own vector to filter out the $f$ most dissimilar gradients in linear time in $q$, and can average out all remaining vectors.
Using a similar proof as for \brute{}, the algorithm thereby defined can be shown to achieve asymptotically the same averaging constant as \brute{}, in the limit $q \gg f$; but it now runs in $\mathcal O(dq)$, (requiring  however $n \geq 7f + 1$).
\end{remark}

\begin{theorem}
\label{th:lower_bound_averaging_constant}
No asynchronous algorithm can achieve better than Byzantine $\frac{h+2f-q}{h}$-averaging agreement.
\end{theorem}

The proof is given in the supplementary material.
It relies on the quasi-unanimity lemma, which shows that if a node receives at least $q -f$ identical vectors $x$, then it must output $x$.
We then construct an instance where, because of this and of Byzantines, honest nodes cannot agree.
Note that, as a corollary, in the regime $q=h \rightarrow \infty$ and $f = o(h)$,
\brute{} achieves asymptotically the best-possible averaging constant, up to a multiplicative constant equal to $3/2$.

\subsection{Optimal Byzantine resilience}
\label{sec:icwtm_alg}

Our second algorithm makes use of reliable broadcast\footnote{Note that \brute{} can also be straightforwardly upgraded into \textsc{RB-MDA} to gain Byzantine resilience.}:  each Byzantine node broadcasts only a single vector (the uniqueness property of reliable broadcast in \cite{abraham2004optimal}).
We denote $\family{}{}{w} \in \mathbb R^{d \cdot n}$ the family of vectors proposed by all nodes.
For each $j \in [h]$, $\indexvar{j}{}{w} = \indexvar{j}{}{x}$ is the vector of an honest node, while a Byzantine node proposes $\indexvar{j}{}{w}$ for each $j \in [h+1,n]$.
Moreover, \cite{abraham2004optimal} showed the existence of a multi-round algorithm which, by using reliable broadcast and a witness mechanism, guarantees that any two honest nodes $j$ and $k$ will collect at least $q$ similar inputs.
Formally, denoting $\quorums{j}{} \subseteq [n]$ the set of nodes whose messages were successfully delivered to node $j$ (including through relays), the algorithm by \cite{abraham2004optimal} guarantees that $\card{\quorums{j}{} \cap \quorums{k}{}} \geq q$ for any two honest nodes $j,k \in [h]$. At each iteration of our \icwtm{} algorithm, each node $j$ exploits the same reliable broadcast and witness mechanism techniques to collect other nodes' vectors.
Now, given its set $\quorums{j}{}$ of collected messages, each node $j$ applies coordinate-wise trimmed mean, denoted \cwtm{}, as follows. For each coordinate $i$, it discards the $f$ smallest $i$-th coordinates it collected, as well as the $f$ largest. We denote $\family{j}{}{z} = \family{Q^{(j)}}{}{w}$ the subfamily received by node $j$, and $S(\family{j}{}{z}[i]) \subset[n]$ the subset of nodes whose $i$-th coordinates remain after trimming. Node $j$ then computes the average $\indexvar{j}{}{y}$ of the $i$-th coordinates of this subset, i.e.
\begin{equation}
    \indexvar{j}{}{y}[i] \triangleq \frac{1}{\card{\quorums{j}{}}-2f} \sum_{k \in S(\family{j}{}{z}[i])} \indexvar{k}{}{w}[i].
\end{equation}

\icwtm{} consists of iterating \cwtm{}, on vectors received from other nodes at each communication round.
Namely, given input $\aggrparameter \in \setn$, \icwtm{} iterates \cwtm{} $T_{\icwtm{}}(\aggrparameter)$ times, where
\begin{equation}
    T_{\icwtm{}}(\aggrparameter) \triangleq \left\lceil \frac{(\aggrparameter+1) \ln 2 + \ln \sqrt{h}}{ \tilde \varepsilon} \right\rceil.
\end{equation}
The correctness of \icwtm{} can then be guaranteed under the following assumption.

\begin{assumption}[Assumption for analysis of \icwtm{}]
\label{ass:query}
There is $\varepsilon > 0$ such that $n \geq (3+\varepsilon) f$. We then set $q = n - f$, and define $\tilde \varepsilon \triangleq \frac{\varepsilon}{1+\varepsilon}$.
\end{assumption}

\begin{theorem}
\label{th:cwtm}
Under Assumption \ref{ass:query}, \icwtm{} guarantees Byzantine $\frac{4f}{\sqrt h}$-averaging agreement.
This is optimal in terms of Byzantine resilience.
Indeed, for $n \leq 3f$, no algorithm can achieve Byzantine averaging agreement.
\end{theorem}

The proof is provided in the supplementary material.
The correctness relies on a coordinate-wise analysis, and on the study of a so-called \emph{coordinate-wise diameter}, and its relation with the $\ell_2$ diameter.
The lower bound exploits the quasi-unanimity lemma.
Note that while \icwtm{} tolerates more Byzantine nodes, its averaging constant is larger than that of \brute{} by a factor of $\mathcal{O}(\sqrt{h})$.

\section{Empirical Evaluation}
\label{sec:eval}

We implemented our collaborative learning algorithms
using \emph{Garfield} library~\cite{guerraoui2021garfield}
and PyTorch~\cite{paszke2019pytorch}.
Each agreement algorithm comes in two variants: one assuming i.i.d. data (See supplementary material) and one tolerating non-i.i.d. data (Algorithm \ref{alg:learn}).
In each case, the first variants require fewer communications.
We report below on the empirical evaluation of the overhead of our four variants when compared to a non--robust collaborative learning approach.
Our baseline is indeed a vanilla fully decentralized implementation in which all nodes
share their updates with each other and then aggregate these updates by \emph{averaging} (a deployment that cannot tolerate even one Byzantine node).

We focus on throughput, measuring the number of updates the system performs per second. As we consider an asynchronous network, we report on the fastest node in each experiment.
We consider image classification tasks, using
MNIST~\cite{mnist} and CIFAR-10~\cite{cifar} datasets.
MNIST is a dataset of handwritten digits with 70,000 $28\times28$ images in 10 classes.
CIFAR-10 consists of 60,000 $32\times32$ colour images in 10 classes.
We use batches of size 100, and we
experimented with 5 models with different sizes ranging from simple models like small convolutional neural network (\emph{MNIST\_CNN} and \emph{Cifarnet}), training a few thousands of parameters, to big models like ResNet-50 with around 23M parameters.
Our experimental platform is Grid5000~\cite{g5k}.
We always employ nodes from the same cluster, each having 2 CPUs (Intel Xeon E5-2630 v4) with 14 cores, 768~GiB RAM, 2$\times$10~Gbps Ethernet, and 2 Nvidia Tesla P100
GPUs.
We set $f=1$, except when deploying our vanilla baseline.

\begin{figure}
\centering
\begin{minipage}{.49\textwidth}
  \centering
  \includegraphics[width=\textwidth]{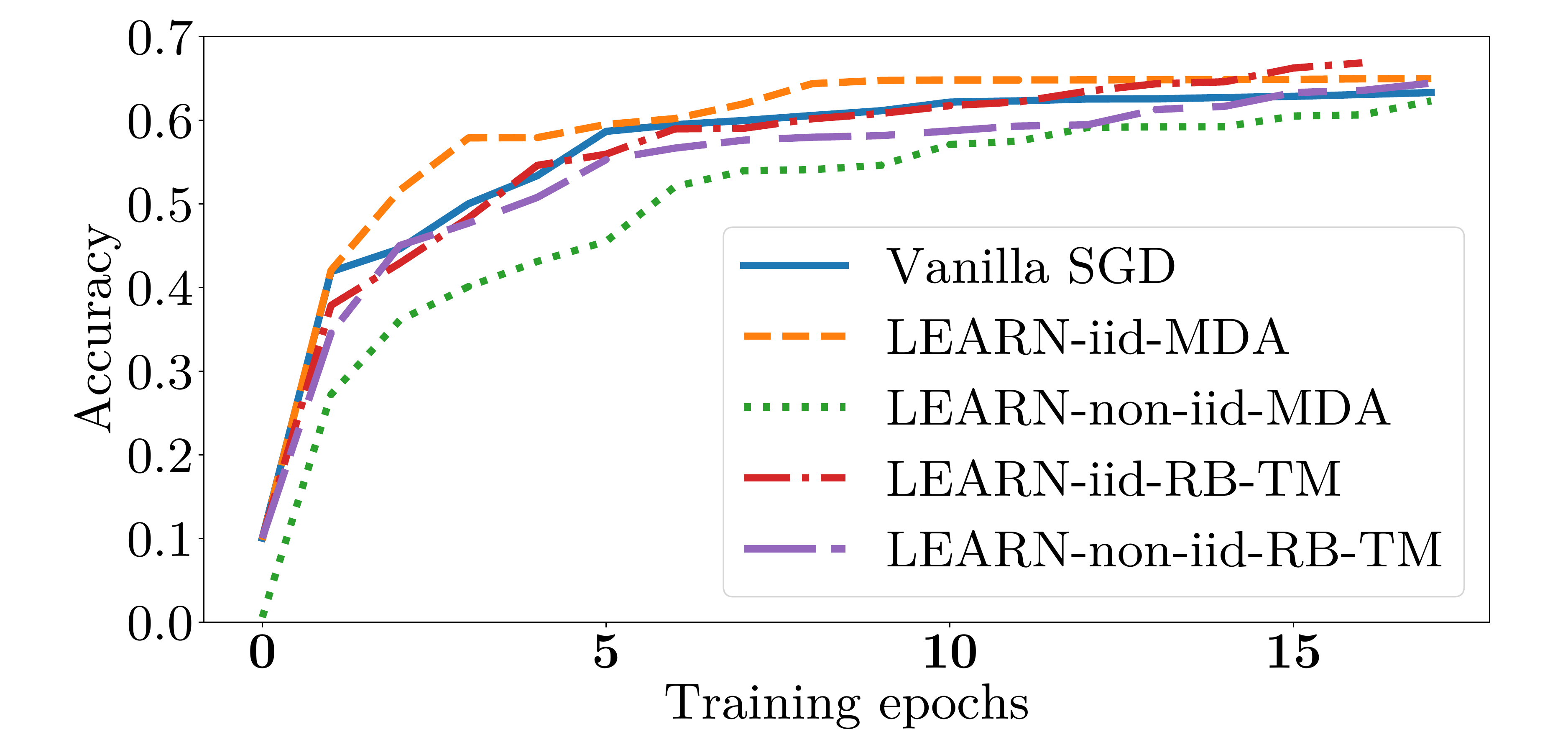}
  \captionof{figure}{Convergence of our algorithms\\ and the vanilla baseline.}
  \label{fig:convergence}
\end{minipage}
\begin{minipage}{.49\textwidth}
  \centering
  \includegraphics[width=\textwidth]{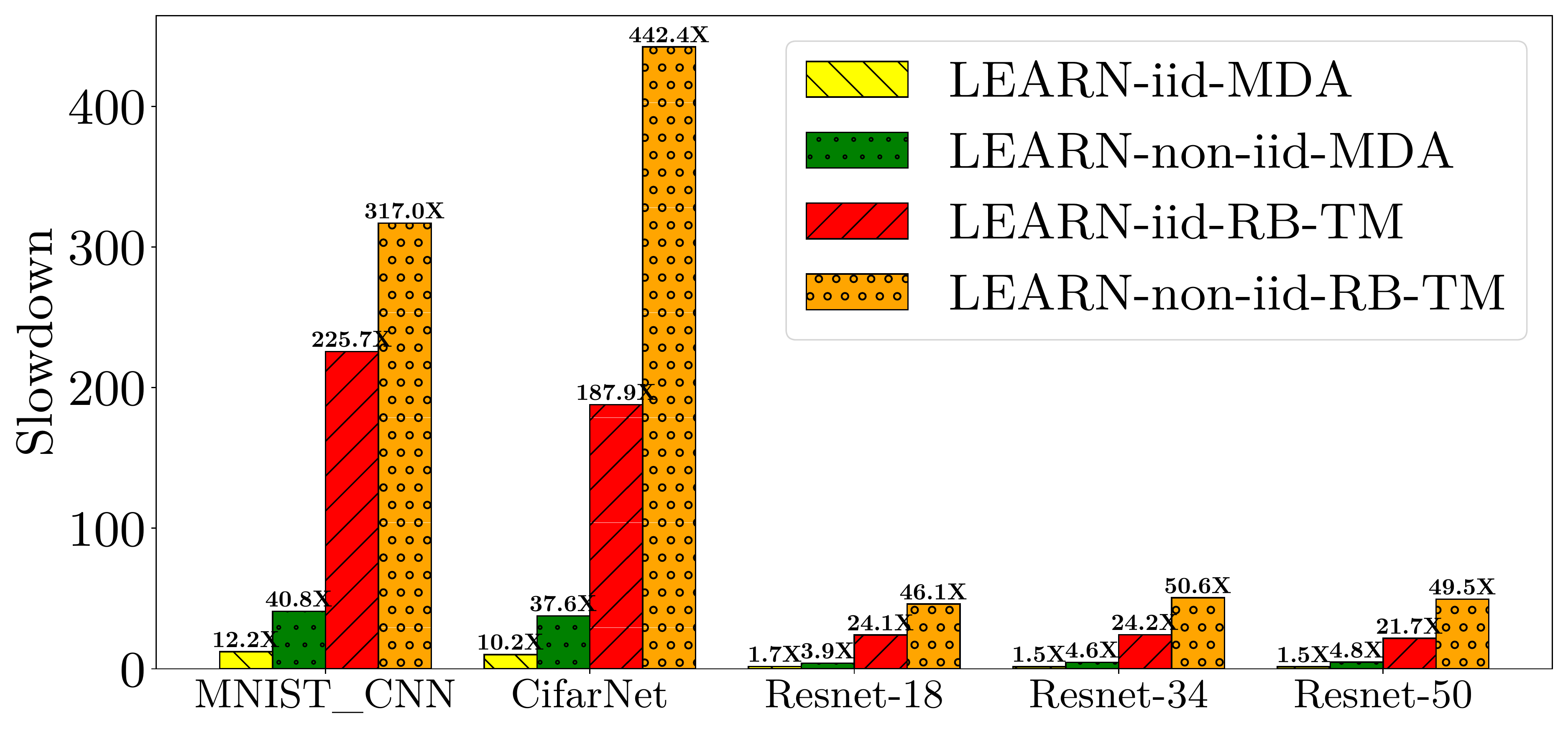}
  \captionof{figure}{Slowdown of our algorithms\\ normalized to the vanilla baseline throughput.}
  \label{fig:slowdown}
\end{minipage}
\end{figure}

Figure~\ref{fig:convergence} compares the convergence of our algorithms to the vanilla baseline w.r.t. the training epochs.
We use 7 nodes in this experiment, and  we train Resnet--18 with CIFAR10.
We verify from this figure that our algorithms can follow the same convergence trajectory as the vanilla baseline.
It is clear from the figure that the \emph{i.i.d.} versions outperform the \emph{non-i.i.d.} ones.

Figure~\ref{fig:slowdown} depicts the throughput overhead of our algorithms (with both i.i.d. and non-i.i.d. data)
compared to our vanilla baseline, with 10 nodes from the same cluster.
Three observations from this figure are in order.
First, the \brute{}--based algorithm performs better than the \icwtm{} one.
The reason is that the latter incurs much more communication messages than the former as the latter uses reliable broadcast and a witness mechanism.
Second, tolerating Byzantine nodes with i.i.d. data is much cheaper than the non-i.i.d. case.
The reason is that it is harder to detect Byzantine behavior when data is not identically distributed on the nodes, which translates into more communication steps.
Third, the slowdown is much higher with small models (i.e.,\ \emph{MNIST\_CNN} and \emph{Cifarnet}).
This is because the network bandwidth is not saturated by the small models in the vanilla case, where it gets congested with the many communication rounds required by our algorithms.
On the other hand, with the larger models, the vanilla deployment saturates the network bandwidth, making the extra communication messages account only for linear overhead.

Finally, it is important to notice that our evaluation is by no means exhaustive and our implementation has not been optimized. Our goal was to give an overview of the relative overheads. With proper optimizations, we believe the actual throughput could be increased for all implementations.
\section{Conclusion}
\label{sec:conc}

We defined and solved collaborative learning in a fully decentralized, Byzantine, heterogeneous, asynchronous and non-convex setting.
We proved that the problem is equivalent to a new abstract form of agreement, which we call averaging agreement.
We then described two solutions to averaging agreement, inducing two original solutions to collaborative learning.
Each solution is optimal along some dimension.
In particular, our lower bounds for the averaging agreement problem provide lower bounds on what any collaborative learning algorithm can achieve.
Such impossibility results would have been challenging to obtain without our reduction.
Our algorithms and our impossibility theorems are very general but can also be adapted for specific settings, such as the presence of a trusted parameter server, the assumption of i.i.d. data or a synchronous context\footnote{In the synchronous case of \brute, with $q = n$, $n \geq \frac{4+2\varepsilon}{1-\varepsilon}f$ is sufficient to guarantee $q \geq \frac{1+\varepsilon}{2} h + \frac{5+3\varepsilon}{2} f$ in Assumption \ref{ass:query_brute}. Also, note that no synchronous algorithm can achieve better than $\frac{f}{h}$-averaging agreement.}.
In the latter case for instance, our two algorithms would only require
$n \geq 4f+1$ and $n \geq 2f+1$, respectively.

\paragraph{{\color{black} Limitations and }potential negative social impacts.} Like all Byzantine learning algorithms, we essentially filter out outliers.
In practice, this may discard minorities with vastly diverging views.
Future research should aim to address this fundamental trade-off between inclusivity and robustness.
We also note that the computation time of $\brute$ grows exponentially with $q$, when $f$ is a constant fraction of $q$.

\begin{ack}
We thank Rafaël Pinot and Nirupam Gupta for their useful comments.
This work has been supported in part by the Swiss National Science Foundation projects: 200021\_182542, Machine learning and 200021\_200477, Controlling the spread of Epidemics.
Most experiments presented in this paper were carried out using the Grid'5000 testbed, supported by a scientific interest group hosted by Inria and including CNRS, RENATER and several Universities as well as other organizations (see \url{https://www.grid5000.fr}).
\end{ack}

\bibliographystyle{plain}
\bibliography{bibliography}
\newpage

\newpage
\appendix
\begin{center}
    \LARGE \bf {Supplementary Material}
\end{center}
In the proofs, we call the first and second properties of averaging agreement ``asymptotic agreement'' and ``C-averaging'' guarantee respectively (see Equation (\ref{eq:avg_agreement}) in the paper).

Moreover, we denote by $\byzfamily{j}{}(\family{}{}{x}) = \left( \indexvar{j,1}{}{\byz}(\family{}{}{x}), \ldots, \indexvar{j,q}{}{\byz}(\family{}{}{x}) \right)$ the family of inputs (of size $q$) collected by node $j$.
There thus exists a bijection $\tau : \quorums{j}{t} \rightarrow [q]$,
where $\quorums{j}{t}$ is the set of nodes that successfully delivered messages to $j$ at round $t$,
such that $\indexvar{j,\tau(k)}{}{\byz} = \indexvar{k}{}{x}$ for all honest nodes $k \in [h]$.

\section{The Equivalence}
\subsection*{Preliminary lemmas}
\begin{lemma}
\label{lemma:diameter_square_unbalanced_inequality}
For any $\alpha>0$ and any two vectors $u$ and $v$, we have
\begin{equation}
    \normtwo{u+v}^2 \leq (1+\alpha^{-1}) \normtwo{u}^2 + (1+\alpha) \normtwo{v}^2.
\end{equation}
As an immediate corollary, for any two families $\family{}{}{u}$ and $\family{}{}{v}$ of vectors, we have
\begin{equation}
    \Delta_2\left(\family{}{}{u}+\family{}{}{v}\right)^2 \leq (1+\alpha^{-1}) \Delta_2\left(\family{}{}{u}\right)^2 + (1+\alpha) \Delta_2\left(\family{}{}{v}\right)^2.
\end{equation}
\end{lemma}

\begin{proof}
  We have the following inequalities:
  \begin{align}
      (1+\alpha^{-1}) \normtwo{u}^2 + (1+\alpha) \normtwo{v}^2 - \normtwo{u+v}^2
      &= \alpha^{-1} \normtwo{u}^2 + \alpha \normtwo{v}^2 - 2 u \cdot v \\
      &= \normtwo{\alpha^{-1/2} u - \alpha^{1/2} v}^2 \geq 0.
  \end{align}
  Rearranging the terms yields the lemma.
\end{proof}

\begin{lemma}
\label{lemma:diameter_square_balanced_inequality}
For any vector family $u_1, \ldots, u_N$, we have
\begin{equation}
    \normtwo{\sum_{j \in [N]} u_j}^2 \leq N \sum_{j \in [N]} \normtwo{u_j}^2.
\end{equation}
As an immediate corollary, for any family of vector families $\family{}{}{u_1}, \ldots, \family{}{}{u_N}$, we have
\begin{equation}
    \Delta_2\left(\sum_{j \in [N]} \family{}{}{u_j} \right)^2 \leq N \sum_{j \in [N]} \Delta_2\left(\family{}{}{u_j}\right)^2.
\end{equation}
\end{lemma}

\begin{proof}
  Notice that $u \mapsto \normtwo{u}^2$ is a convex function. As a result,
  \begin{equation}
      \normtwo{\frac{1}{N} \sum_{j \in [N]} u_j}^2 \leq \frac{1}{N} \sum_{j \in [N]} \normtwo{u_j}^2.
  \end{equation}
  Multiplying both sides by $N^2$ allows to conclude.
\end{proof}

\begin{lemma}
\label{lemma:diameter_of_small_vectors}
For any vector family $\family{}{}{u} \in \mathbb R^{d \cdot h}$, we have
\begin{equation}
    \Delta_2(\family{}{}{u}) \leq 2 \max_{j \in [h]} \normtwo{\indexvar{j}{}{u}}.
\end{equation}
\end{lemma}

\begin{proof}
  We have the inequalities
  \begin{align}
      \Delta_2(\family{}{}{u}) &= \max_{j,k \in [h]} \normtwo{\indexvar{j}{}{u} - \indexvar{k}{}{u}}
      \leq \max_{j,k \in [h]} \normtwo{\indexvar{j}{}{u}} + \normtwo{\indexvar{k}{}{u}} \\
      &= \max_{j \in [h]} \normtwo{\indexvar{j}{}{u}} + \max_{k \in [h]} \normtwo{\indexvar{k}{}{u}}
      = 2 \max_{j \in [h]} \normtwo{\indexvar{j}{}{u}},
  \end{align}
  which is the lemma.
\end{proof}

We now prove that \learn{} solves collaborative learning. Note that all the proofs depend on some quantity $\alpha_t$, which will eventually be defined as $\alpha_t \triangleq \max \left\lbrace 1/\sqrt{t}, \sigma_t \right\rbrace$. Note also that we then have $\alpha_t \leq \bar \alpha \triangleq \max \left\lbrace 1, \sigma \right\rbrace$ since $b_t \geq 1$. We also define
\begin{equation}
\label{eq:error_def}
  \noises{j}{t} \triangleq \grads{j}{t} - \nabla \loss^{(j)}(\params{j}{t}),
\end{equation}
the gradient estimation error (noise) of node $j$ at round $t$, whose norm is bounded by $\sigma_t$ (see Section \ref{sec:assumptions}).

\begin{lemma}
\label{lemma:guanyu_gradient_bound}
  Under assumptions (\ref{ass:lipschitz}, \ref{ass:variance}), for any $0 < \alpha_t \leq \bar \alpha$,
  we have the following bound on the expected $\ell_2$ diameter of gradients:
  \begin{equation}
    \expect_{\family{}{t}{\noise} | \family{}{t}{\param}} \Delta_2 \left( \family{}{t}{\grad} \right)^2
    \leq (1+\alpha_t) \bgd^2 + 16 \bar \alpha \alpha_t^{-1} \left( \lipschitz^2 \Delta_2 ( \family{}{t}{\param} )^2 +  h \sigma_t^2 \right).
  \end{equation}
\end{lemma}

\begin{proof}
  Note that we have
  \begin{align}
      \grads{j}{t} &= \localgrad{j}{\params{j}{t}} + \noises{j}{t} \\
      &=  \localgrad{j}{\avgparam{t}} +\left( \left(\localgrad{j}{\params{j}{t}} - \localgrad{j}{\avgparam{t}} \right) +  \noises{j}{t} \right).
  \end{align}
  Applying Lemma \ref{lemma:diameter_square_unbalanced_inequality} with $\alpha = \alpha_t^{-1}$, and then Lemma \ref{lemma:diameter_square_balanced_inequality} to the last terms then yields
  \begin{align}
      \Delta_2 \left( \family{}{t}{\grad} \right)^2
      &\leq (1+\alpha_t) \Delta_2 \left( \localgradfamily \left(\avgparam{t}\right) \right)^2 \\
      &+ (1+\alpha_t^{-1}) \left( 2\Delta_2 \left( \localgradfamily{}\left( \family{}{t}{\param} \right) - \localgradfamily{} \left( \avgparam{t} \right) \right)^2 + 2 \Delta_2 \left( \family{}{t}{\noise} \right)^2 \right).
  \end{align}
  Note that
  \begin{align}
      \Delta_2 \left( \localgradfamily \left(\avgparam{t}\right) \right)^2
      &= \max_{j,k \in [h]} \normtwo{
      \localgrad{j}{\avgparam{t}} - \localgrad{k}{\avgparam{t}} }^2 \leq K^2.
  \end{align}
  The second term can be controlled using Lemma \ref{lemma:diameter_of_small_vectors}, which yields
  \begin{align}
      \Delta_2 \left( \localgradfamily{} \left(\family{}{t}{\param}\right) - \localgradfamily{} \left( \avgparam{t} \right) \right)
      &\leq 2 \max_{j \in [h]} \normtwo{\localgrad{j}{\params{j}{t}} - \localgrad{j}{\avgparam{t}}} \\
      &\leq 2 \max_{j \in [h]} \lipschitz \normtwo{\params{j}{t} - \avgparam{t}} \leq 2 \lipschitz \Delta_2(\family{}{t}{\param}).
  \end{align}
  To bound the third term, first note that Lemma \ref{lemma:diameter_of_small_vectors} implies that $\Delta_2(\family{}{t}{\noise}) \leq 2 \max_{j \in [h]} \normtwo{\noises{j}{t}}$.
  Thus,
  \begin{align}
      \expect_{\family{}{t}{\noise} | \family{}{t}{\param}} \Delta_2(\family{}{t}{\noise})^2
      &\leq 4 \expect_{\family{}{t}{\noise} | \family{}{t}{\param}} \max_{j \in [h]} \normtwo{\noises{j}{t}}^2
      \leq 4 \expect_{\family{}{t}{\noise} | \family{}{t}{\param}} \sum_{j \in [h]} \normtwo{\noises{j}{t}}^2 \\
      &= 4h \expect_{\family{}{t}{\noise} | \family{}{t}{\param}} \normtwo{\noises{j}{t}}^2
      \leq 4h \sigma_t^2.
  \end{align}
  Combining it all, and using $1+\alpha_t^{-1} \leq (\bar \alpha+1) \alpha_t^{-1} \leq 2\bar \alpha \alpha_t^{-1}$ for $\alpha_t \leq \bar \alpha$ and $\bar \alpha = \max \left\lbrace 1, \sigma \right\rbrace \geq 1$, yields the result.
\end{proof}

\begin{definition}
The (stochastic) effective gradient $\effectgrads{j}{t}$ of node $j$ is defined by
\begin{equation}
  \effectgrads{j}{t} = \frac{\params{j}{t} - \params{j}{t+1}}{\eta}.
\end{equation}
\end{definition}
In particular, we shall focus on the effective gradient of the average parameter, which turns out to also be the average of the effective gradients, that is,
\begin{equation}
\label{eq:avg_eff_grad}
    \avgeffectgrad{t} \triangleq \frac{1}{h} \sum_{j \in [h]} \effectgrads{j}{t}
    = \frac{\avgparam{t} - \avgparam{t+1}}{\eta}.
\end{equation}

\begin{lemma}
\label{lemma:guanyu_effective_gradient_square}
  Under assumptions (\ref{ass:lipschitz}, \ref{ass:variance}), for any $0 < \alpha_t \leq \bar \alpha$, the expected discrepancy between the average effective gradient and the true gradient at the average parameter is bounded as follows:
  \begin{align}
    &\expect_{\family{}{t}{\noise}, \aggr{}|\family{}{t}{\param}} \normtwo{\avgeffectgrad{t} - \avglocalgrad{\avgparam{t}}}^2 \leq
    (1+\alpha_t) C^2 \expect_{\family{}{t}{\noise}|\family{}{t}{\param}} \Delta_2 \left(\family{}{t}{\grad}\right)^2 \nonumber \\
    &\qquad \qquad \qquad + 6\bar \alpha \alpha_t^{-1} \left[ \frac{\sigma_t^2}{h} + \left( \lipschitz^2 + \frac{2 C^2}{\eta^2} \right) \Delta_2\left( \family{}{t}{\param} \right)^2 + \frac{2 C^2}{t^2} \expect_{\family{}{t}{\noise}|\family{}{t}{\param}} \Delta_2 \left(\family{}{t}{\grad}\right)^2  \right].
  \end{align}
\end{lemma}

\begin{proof}
  Note that
  \begin{align}
   \avgparam{t+1} - \avgparam{t} &= (\avgparam{t+1} - \avgparam{t+1/2}) + (\avgparam{t+1/2}-\avgparam{t}) \\
   &= (\avgparam{t+1} - \avgparam{t+1/2}) - \eta \avgcwtmgrad{t}.
  \end{align}
  As a result $\avgeffectgrad{t} = \avggrad{t} + \left( \avgcwtmgrad{t} - \avggrad{t} \right) - \frac{1}{\eta} (\avgparam{t+1} - \avgparam{t+1/2})$. Moreover, we have
  \begin{align}
    \avggrad{t} &= \frac{1}{h} \sum_{j \in [h]} \localgrad{j}{\params{j}{t}} + \frac{1}{h} \sum_{j \in [h]} \noises{j}{t} \\
    &= \realgrad{\avgparam{t}} + \frac{1}{h} \sum_{j \in [h]} \left(\localgrad{j}{\params{j}{t}} - \localgrad{j}{\avgparam{t}} \right) + \frac{1}{h} \sum_{j \in [h]} \noises{j}{t},
  \end{align}
  where $\realgrad{\avgparam{t}} = \frac{1}{h} \sum_{j \in [h]} \localgrad{j}{\avgparam{t}}$ is the average gradient at the average parameter. This then yields:
\begin{align}
  \avgeffectgrad{t} - \realgrad{\avgparam{t}} &= \frac{1}{h} \sum_{j \in [h]} \left( \localgrad{j}{\params{j}{t}} - \localgrad{j}{\avgparam{t}} \right) + \frac{1}{h} \sum_{j \in [h]} \noises{j}{t} \nonumber \\
  &\qquad \qquad \qquad  \qquad + \frac{1}{\eta} \left( \avgparam{t+1/2} - \avgparam{t+1} \right) + \left( \avgcwtmgrad{t} - \avggrad{t} \right).
\end{align}
Applying Lemma \ref{lemma:diameter_square_unbalanced_inequality} for $\alpha = \alpha_t$ (by isolating the first three terms), and then Lemma \ref{lemma:diameter_square_balanced_inequality} to the first three terms then yields
\begin{align}
  \normtwo{\avgeffectgrad{t} - \realgrad{\avgparam{t}}}^2 &\leq 3(1+\alpha_t^{-1}) \normtwo{\frac{1}{h} \sum_{j \in [h]} \left( \localgrad{j}{\params{j}{t}} - \localgrad{j}{\avgparam{t}} \right)}^2 \nonumber \\
  &+ 3(1+\alpha_t^{-1}) \normtwo{\frac{1}{h} \sum_{j \in [h]} \noises{j}{t}}^2
  + \frac{3(1+\alpha_t^{-1})}{\eta^2} \normtwo{ \avgparam{t+1/2} - \avgparam{t+1} }^2  \nonumber \\
  &+ (1+\alpha_t) \normtwo{ \avgcwtmgrad{t} - \avggrad{t} }^2.
\end{align}

We now note that the expectation of each term can be bounded. Indeed,
\begin{align}
    &\normtwo{\frac{1}{h} \sum_{j \in [h]} \left( \localgrad{j}{\params{j}{t}} - \localgrad{j}{\avgparam{t}} \right)} \leq \frac{1}{h} \sum_{j \in [h]} \normtwo{\localgrad{j}{\params{j}{t}} - \localgrad{j}{\avgparam{t}}} \\
    &\qquad \qquad \qquad \leq \frac{1}{h} \sum_{j \in [h]} \lipschitz \normtwo{\params{j}{t} - \avgparam{t}}
    \leq \frac{1}{h} \sum_{j \in [h]} \lipschitz \Delta_2(\family{}{t}{\param}) = \lipschitz \Delta_2(\family{}{t}{\param}).
\end{align}
Moreover, using the conditional non-correlation of $\noises{j}{t}$, we have
\begin{align}
    \expect_{\noises{}{t}|\family{}{t}{\param}} \normtwo{\frac{1}{h} \sum_{j \in [h]} \noises{j}{t}}^2
    &= \expect_{\noises{}{t}|\family{}{t}{\param}} \frac{1}{h^2} \sum_{j,k \in [h]} \noises{j}{t} \cdot \noises{k}{t}
    = \frac{1}{h^2} \sum_{j,k \in [h]} \expect_{\noises{}{t}|\family{}{t}{\param}} \noises{j}{t} \cdot \noises{k}{t} \\
    &= \frac{1}{h^2} \sum_{j \in [h]} \expect_{\noises{}{t}|\family{}{t}{\param}} \normtwo{\noises{j}{t}}^2
    \leq \frac{1}{h^2} \sum_{j \in [h]} \sigma_t^2 = \frac{\sigma_t^2}{h}.
\end{align}
For the third term, we use the $C$-averaging guarantee of \aggr{} to obtain
\begin{equation}
    \expect_{\aggr{}|\family{}{t+1/2}{\param}} \normtwo{\avgparam{t+1} - \avgparam{t+1/2}}^2 \leq C^2 \Delta_2(\family{}{t+1/2}{\param})^2.
\end{equation}
Since $\family{}{t+1/2}{\param} = \family{}{t}{\param} - \eta \family{}{t}{\cwtmgrad}$, using Lemma \ref{lemma:diameter_square_balanced_inequality} and taking the expectation over $\family{}{t}{\noise}$, we then have
\begin{align}
    \expect_{\family{}{t}{\noise},\aggr{}|\family{}{t}{\param}} \normtwo{\avgparam{t+1} - \avgparam{t+1/2}}^2
    &\leq 2 C^2 \Delta_2(\family{}{t}{\param})^2 +  2C^2  \eta^2 \expect_{\family{}{t}{\noise},\aggr{}|\family{}{t}{\param}} \Delta_2(\family{}{t}{\cwtmgrad})^2 \\
    &\leq 2 C^2  \Delta_2(\family{}{t}{\param})^2 + 2C^2 \eta^2 \frac{\expect_{\family{}{t}{\noise}|\family{}{t}{\param}}\Delta_2(\family{}{t}{\grad})^2}{t^2}.
\end{align}
Finally, for the last term, we again use the $C$-averaging guarantee of the aggregation \aggr{} and take the expectation over $\family{}{t}{\noise}$:
\begin{equation}
    \expect_{\family{}{t}{\noise},\aggr{}|\family{}{t}{\param}} \normtwo{\avgcwtmgrad{t} - \avggrad{t}}^2 \leq C^2 \expect_{\family{}{t}{\noise}|\family{}{t}{\param}} \Delta_2(\family{}{t}{\grad})^2.
\end{equation}
Combining it all, and using $1+\alpha_t^{-1} \leq 2 \bar \alpha \alpha_t^{-1}$ finally yields the lemma.
\end{proof}

\begin{lemma}
\label{lemma:guanyu_discrepancy_effective_gradient}
  Under assumptions (\ref{ass:lipschitz}, \ref{ass:variance}), for any $0 < \alpha_t \leq \bar \alpha$ and $\alpha_t \geq 1/\sqrt{t}$, there exist constants $A$ and $B$ which can be computed explicitly given $\bar \alpha$, $C$, $\lipschitz$, $h$ and $\eta$, such that
  \begin{equation}
      \expect_{\family{}{t}{\noise},\aggr{}|\family{}{t}{\param}} \normtwo{\avgeffectgrad{t} - \avglocalgrad{\avgparam{t}}}^2
      \leq (1+\alpha_t)^2 \left( 1+ \kappa_t \right) C^2 K^2 + \alpha_t^{-1} \left( A \Delta_2(\family{}{t}{\param})^2 + B \sigma_t^2 \right),
  \end{equation}
  where $\kappa_t \leq \frac{12 \bar \alpha}{ t^{(3/2)}}$.
\end{lemma}

\begin{proof}
  Combining the two previous lemmas, this bound can be guaranteed by setting
  \begin{align}
      \kappa_t &= \frac{12 \bar \alpha}{(1+\alpha_t) \alpha_t t^2} \\
      A_t &= 16(1+\alpha_t)\bar \alpha C^2 \lipschitz^2 + 6 \bar \alpha \lipschitz^2 + \frac{12\bar \alpha C^2}{\eta^2} + \frac{192 \bar \alpha^2 C^2 \lipschitz^2}{\alpha_t t^2} \\
      B_t &= 16(1+\alpha_t) \bar \alpha C^2 h + \frac{6\bar \alpha}{h} + \frac{192 \bar \alpha^2 C^2 h}{\alpha_t t^2}.
  \end{align}
  Assumptions $0 < \alpha_t \leq \bar \alpha$ (which implies $1+\alpha_t \leq 2 \bar \alpha$) and $ \alpha_t \sqrt{t} \geq 1$ allow to conclude, with
  \begin{align}
       \kappa_t &\leq \frac{12 \bar \alpha}{ t^{(3/2)}}\\
       A &= 32\bar \alpha^2  C^2 \lipschitz^2 + 6\bar \alpha \lipschitz^2 + \frac{12\bar\alpha C^2}{\eta^2} + 192 \bar \alpha^2 C^2 \lipschitz^2 \\
       B &= 32\bar \alpha^2 C^2 h + \frac{6\bar \alpha}{h} + 192 \bar \alpha^2 C^2 h.
  \end{align}
  This shows in particular that $A$ and $B$ can indeed be computed from the different constants of the problem.
\end{proof}

\begin{lemma}
\label{lemma:guanyu_drift}
We have the following bound on parameter drift:
\begin{equation}
    \expect_{\aggr{} | \family{}{t}{\param}, \family{}{t}{g}} \Delta_2 (\family{}{t+1}{\param})^2
    \leq \frac{1}{2} \Delta_2 (\family{}{t}{\param})^2
    + \frac{\eta^2}{2 t^2} \Delta_2 (\family{}{t}{\grad})^2.
\end{equation}
\end{lemma}

\begin{proof}
  Recall that $\family{}{t+1}{\param} = \aggrfamily{}{1} \circ \byzfamily{}{t,\param} (\family{}{t+1/2}{\param})$. Thus, by the asymptotic agreement property of $\aggr{}_1$, we know that
  \begin{equation}
    \expect_{\aggr{} | \family{}{t}{\param}, \family{}{t}{g}} \Delta_2(\family{}{t+1}{\param})^2\leq \frac{1}{4} \Delta_2 (\family{}{t+1/2}{\param})^2.
  \end{equation}
  Now recall that $\family{}{t+1/2}{\param} = \family{}{t}{\param}-\eta \family{}{t}{\cwtmgrad}$. Applying Lemma \ref{lemma:diameter_square_unbalanced_inequality} for $\alpha = 1$ thus yields
  \begin{equation}
      \Delta_2(\family{}{t+1/2}{\param})^2 \leq 2 \Delta_2(\family{}{t}{\param})^2 + 2 \eta^2 \Delta_2(\family{}{t}{\cwtmgrad})^2.
  \end{equation}
  We now use the asymptotic agreement property of $\aggr{}_{\aggrparameter(t)}$, which yields
  \begin{equation}
      \expect_{\aggr{} | \family{}{t}{\param}, \family{}{t}{g}} \Delta_2(\family{}{t}{\cwtmgrad})^2 \leq \frac{\Delta_2(\family{}{t}{\grad})^2}{t^2}.
  \end{equation}
  Combining it all then yields the result.
\end{proof}

\begin{lemma}
\label{lemma:guanyu_vanishing_diameter}
  Under assumptions (\ref{ass:lipschitz}, \ref{ass:variance}), $0 < \alpha_t \leq \bar \alpha$ and $\alpha_t = \max \left\lbrace 1/\sqrt{t}, \sigma_t \right\rbrace$, there exists a constant $D$ such that
  \begin{equation}
      \expect_{\family{}{1:t}{\noise}} \Delta_2(\family{}{t}{\param})^2
      \leq D/t^2.
  \end{equation}
  Note that the constant $D$ can be computed from the constants $\lipschitz$, $\eta$, $\bgd$, $h$ and the functions $\sigma_t$ and $\alpha_t$.
\end{lemma}

\begin{proof}
  Denote $u_t = \expect_{\family{}{1:t}{\noise}} \Delta_2(\family{}{t}{\param})^2$. Combining Lemmas \ref{lemma:guanyu_gradient_bound} and \ref{lemma:guanyu_drift} yields
  \begin{equation}
      u_{t+1} \leq \rho_t u_t + \delta_t,
  \end{equation}
  where $\rho_t$ and $\delta_t$ are given by
  \begin{align}
      \rho_t &\triangleq \frac{1}{2} + \frac{8  \bar \alpha \lipschitz^2 \eta^2}{\alpha_t t^2} \\
      \delta_t &\triangleq \frac{\eta^2}{2 t^2} \left( (1+\alpha_t) \bgd^2 + 16 \bar \alpha \alpha_t^{-1} h \sigma_t^2 \right).
  \end{align}
  Given that $\alpha_t \geq 1/\sqrt{t}$, we know that $\rho_t \leq \frac{1}{2} + 8\bar \alpha \lipschitz^2 \eta^2 t^{-3/2}$. Thus, for $t \geq t_0 \triangleq \left( 32 \bar \alpha \lipschitz^2 \eta^2 \right)^{2/3}$, we know that $\rho_t \leq \rho \triangleq 3/4$. Moreover, using now $\sigma_t \leq \alpha_t \leq \bar \alpha$, we know that
  \begin{align}
      \delta_t &\leq \frac{\eta^2}{2t^2} \left( (1+\bar \alpha) K^2 + 16 \bar \alpha h \sigma_t \right) \leq \frac{\eta^2}{2t^2} \left( (1+\bar \alpha) \bgd^2 + 16 h \bar \alpha^2 \right) \triangleq \delta_t^+.
  \end{align}
  Now note that $\delta_t^+$ is decreasing. In particular, for $t \geq t_0$ we now have
  \begin{equation}
      u_{t+1} \leq \rho u_t + \delta^+_t.
  \end{equation}
  By induction we see that, for $t \geq 1$, we have
  \begin{equation}
      u_{t+t_0} \leq \rho^{t} u_{t_0} + \sum_{s = 0}^{t-1} \rho^s \delta^+_{t+t_0-s-1}.
  \end{equation}
  We now separate the sum into two parts. Calling $t_1$ the separation point for $t_1 \geq 1$, and using the fact that $\delta^+_t$ is decreasing yields
  \begin{align}
      u_{t+t_0} &\leq \rho^{t} u_{t_0} + \sum_{s = 0}^{t_1-1} \rho^s \delta^+_{t+t_0-s-1}  + \sum_{s = t_1}^{t-1} \rho^s \delta^+_{t+t_0-s-1} \\
      &\leq \rho^{t} u_{t_0} + \delta^+_{t+t_0 - t_1} \sum_{s = 0}^{t_1-1} \rho^s + \delta^+_{t_0} \sum_{s = t_1}^{t-1} \rho^s \\
      &\leq \rho^t u_{t_0} + \delta^+_{t+t_0 - t_1} \sum_{s = 0}^{\infty} \rho^s + \delta^+_{t_0} \sum_{s = t_1}^{\infty} \rho^s \\
      &\leq \rho^t u_{t_0} + \frac{\delta^+_{t+t_0 - t_1}}{1-\rho} + \frac{\rho^{t_1} \delta^+_{t_0}}{1-\rho} \\
      &= \rho^t u_{t_0} + 4 \delta^+_{t+t_0 - t_1} + 4  \rho^{t_1} \delta^+_{t_0}.
  \end{align}
  We now take $t_1 = \left\lfloor \frac{t+t_0}{2} \right\rfloor$. As a result,
  \begin{equation}
  \label{eq:implicit_definition_of_D}
      \delta^+_{t+t_0-t_1} = \delta^+_{ \left\lceil \frac{t+t_0}{2} \right\rceil} \leq \frac{2 \eta^2}{(t+t_0)^2} \left( (1+\bar \alpha) \bgd^2 + 16 h\bar \alpha^2 \right).
  \end{equation}
  Now define $v_t$ by $v_0 = 0$ and $v_{t+1} = \rho_t v_t + \delta_t$. Note that $u_{t_0}$ can be upper-bounded given $\lipschitz$, $\eta$, $\alpha_{1:t_0}$, $\bgd$, $h$ and $\sigma_{1:t}$, by computing $v_{t_0}$. Indeed, by induction we then clearly have $u_{t_0} \leq v_{t_0}$, and thus the bound
  \begin{equation}
      u_{t+t_0} \leq \frac{8 \eta^2 \left((1+\bar \alpha) \bgd^2 + 16 h\bar \alpha^2 \right)}{(t+t_0)^2} + \rho^t v_{t_0} + 4 \rho^{t_1} \delta_{t_0}^+,
  \end{equation}
  where the right-hand side is perfectly computable given the constants of the problem. Given that $\rho^{t_1} = \mathcal{O}(1/t^2)$ and $\rho^{t} = \mathcal{O}(1/t^2)$, we can also compute a constant $D$ from these constants, such that for all iterations $t$, we have $u_t \leq D/t^2$.
\end{proof}

\subsection{Reduction from collaborative learning to averaging agreement}
Now we proceed with the proof of our theorem.

\begin{proof}[Proof of Theorem \ref{th:gradient_near_convergence}]
  At any iteration $t$, Taylor's theorem implies the existence of $\lambda \in [0,1]$ such that
  \begin{align}
      \avgloss{\avgparam{t+1}} &= \avgloss{\avgparam{t} - \eta \avgeffectgrad{t}} \\
      &= \avgloss{\avgparam{t}} - \eta \avgeffectgrad{t} \cdot \avglocalgrad{\avgparam{t}} + \frac{1}{2} \left( \eta \avgeffectgrad{t} \right)^T \nabla^2 \avgloss{\avgparam{t} - \lambda \eta \avgeffectgrad{t}} \left( \eta \avgeffectgrad{t} \right).
  \end{align}
  Lipschitz continuity of the gradient implies that $\nabla^2 \avgloss{\avgparam{t} - \lambda \eta \avgeffectgrad{t}} \preceq \lipschitz I$, which thus implies
  \begin{align}
      \avgloss{\avgparam{t+1}}
      &\leq \avgloss{\avgparam{t}} - \eta \avgeffectgrad{t} \cdot \avglocalgrad{\avgparam{t}} + \frac{\lipschitz \eta^2}{2} \normtwo{\avgeffectgrad{t}}^2.
      \label{eq:taylor_bound}
  \end{align}
  For the second term, using the inequality $2 u \cdot v \geq - \normtwo{u}^2 - \normtwo{v}^2$, note that
  \begin{align}
      \avgeffectgrad{t} \cdot \avglocalgrad{\avgparam{t}}
      &= \left( \avgeffectgrad{t} - \avglocalgrad{\avgparam{t}} + \avglocalgrad{\avgparam{t}} \right) \cdot \avglocalgrad{\avgparam{t}} \\
      &= \left( \avgeffectgrad{t} - \avglocalgrad{\avgparam{t}} \right) \cdot \avglocalgrad{\avgparam{t}} + \normtwo{\avglocalgrad{\avgparam{t}}}^2 \\
      &\geq - \frac{1}{2} \normtwo{ \avgeffectgrad{t} - \avglocalgrad{\avgparam{t}}}^2 - \frac{1}{2} \normtwo{\avglocalgrad{\avgparam{t}}}^2 + \normtwo{\avglocalgrad{\avgparam{t}}}^2 \\
      &= - \frac{1}{2} \normtwo{ \avgeffectgrad{t} - \avglocalgrad{\avgparam{t}}}^2 + \frac{1}{2} \normtwo{\avglocalgrad{\avgparam{t}}}^2.
  \end{align}
  For the last term, we use $\normtwo{a+b}^2 \leq 2\normtwo{a}^2+2\normtwo{b}^2$ to derive
  \begin{align}
      \normtwo{\avgeffectgrad{t}}^2
      &= \normtwo{ \avgeffectgrad{t} - \avglocalgrad{\avgparam{t}} + \avglocalgrad{\avgparam{t}} }^2 \\
      &\leq 2 \normtwo{\avgeffectgrad{t} - \avglocalgrad{\avgparam{t}}}^2 + 2 \normtwo{\avglocalgrad{\avgparam{t}}}^2.
  \end{align}
  Combining the two above bounds into Equation (\ref{eq:taylor_bound}) yields
  \begin{equation}
      \avgloss{\avgparam{t+1}}
      \leq \avgloss{\avgparam{t}} - \left( \frac{\eta}{2} - \lipschitz \eta^2 \right) \normtwo{\avglocalgrad{\avgparam{t}}}^2 + \left( \frac{\eta}{2} + \lipschitz \eta^2 \right) \normtwo{\avgeffectgrad{t} - \avglocalgrad{\avgparam{t}}}^2.
  \end{equation}
  Rearranging the terms then yields
  \begin{equation}
      \left( \frac{\eta}{2} - \lipschitz \eta^2 \right) \normtwo{\avglocalgrad{\avgparam{t}}}^2 \leq \avgloss{\avgparam{t}} - \avgloss{\avgparam{t+1}} + \left( \frac{\eta}{2} + \lipschitz \eta^2 \right) \normtwo{\avgeffectgrad{t} - \avglocalgrad{\avgparam{t}}}^2.
  \end{equation}
  We now use the fact that $\eta \leq \delta/12\lipschitz$. Denoting $\nu \triangleq \delta/6$, this implies that $\frac{\eta}{2} - \lipschitz \eta^2 \geq (1-\nu) \eta/2$ and $\frac{\eta}{2} + \lipschitz \eta^2 \leq (1+\nu) \eta/2$. As a result,
  \begin{equation}
      \normtwo{\avglocalgrad{\avgparam{t}}}^2 \leq \frac{4}{\eta} \left( \avgloss{\avgparam{t}} - \avgloss{\avgparam{t+1}}\right) + \frac{1+\nu}{1-\nu} \normtwo{\avgeffectgrad{t} - \avglocalgrad{\avgparam{t}}}^2.
  \end{equation}
  Taking the expectation and the average over $t \in [T]$ yields
  \begin{equation}
      \expect_{\family{}{1:T}{\noise}} \frac{1}{T} \sum_{t \in [T]}  \normtwo{\avglocalgrad{\avgparam{t}}}^2 \leq \frac{4 (\avgloss{\avgparam{1}}- \avgloss{\avgparam{T+1}})}{\eta T} + \frac{1+\nu}{1-\nu} \frac{1}{T} \sum_{t \in [T]} \expect_{\family{}{1:T}{\noise}} \normtwo{\avgeffectgrad{t} - \avglocalgrad{\avgparam{t}}}^2.
  \end{equation}
  Note that $\expect_{\family{}{1:T}{\noise}} \frac{1}{T} \sum_{t \in [T]}  \normtwo{\avglocalgrad{\avgparam{t}}}^2 = \expect \normtwo{\avglocalgrad{\avgparam{*}}}^2$, since the second term is obtained by taking uniformly randomly one of the values averaged in the first term. Using also the fact that $\avgloss{\avgparam{T+1}} \geq \inf_{\param} \avgloss{\param} \geq 0$ (Assumption \ref{ass:nonnegative}) and $\loss^{(j)}(\indexvar{}{1}{\param}) \leq \loss_{max}$, we then obtain
  \begin{equation}
  \label{eq:guanyu_first_bound}
      \expect \normtwo{\avglocalgrad{\avgparam{*}}}^2 \leq \frac{4 \loss_{max}}{\eta T} + \frac{1+\nu}{1-\nu} \frac{1}{T} \sum_{t \in [T]} \expect_{\family{}{1:T}{\noise}} \normtwo{\avgeffectgrad{t} - \avglocalgrad{\avgparam{t}}}^2.
  \end{equation}
  Now recall that all nodes started with the same value $\param_1$, and thus know $\avgparam{1}$. As a result, each node $j$ can compute $\loss^{(j)}(\avgparam{1})$, but it cannot compute $\avgloss{\avgparam{1}}$.

  Let us focus on the last term. Taking Lemma \ref{lemma:guanyu_discrepancy_effective_gradient} and averaging over all noises $\family{}{1:t}{\noise}$ yields
  \begin{equation}
      \expect_{\family{}{1:t}{\noise}} \normtwo{\avgeffectgrad{t} - \avglocalgrad{\avgparam{t}}}^2
      \leq (1+\alpha_t)^2 \left( 1+ \kappa_t \right) C^2 K^2 + \alpha_t^{-1} \left( A \expect_{\family{}{1:t}{\noise}} \Delta_2(\family{}{t}{\param})^2 + B \sigma_t^2 \right).
  \end{equation}
  Recall that, by Lemma \ref{lemma:guanyu_vanishing_diameter}, $\expect_{\family{}{1:t}{\noise}} \Delta_2(\family{}{t}{\param})^2 \leq D/t^2$. Recall also that $\alpha_t \triangleq \max \left\lbrace 1/\sqrt{t}, \sigma_t \right\rbrace$ and $\kappa_t \leq {12 \bar \alpha}/{ t^{(3/2)}}$. Then we obtain
  \begin{equation}
      \expect_{\family{}{1:t}{\noise}} \normtwo{\avgeffectgrad{t} - \avglocalgrad{\avgparam{t}}}^2
      \leq \left( 1+ \alpha_t \right)^2 \left( 1+\frac{12 \bar \alpha}{ t^{3/2}} \right) C^2 K^2 + \frac{AD}{t^{3/2}} + B\sigma_t.
  \end{equation}
  Now recall that we regularly increase the batch size (see Section \ref{sec:assumptions}). Thus, there exists some iteration $T_1$, such that $\sigma_{T_1} \leq \min \left\lbrace \nu, \nu C^2 \bgd^2 /8B \right\rbrace = \mathcal O(\delta)$.
  Defining $T_2 \triangleq 1/\nu^2$, $T_3 \triangleq (12\bar\alpha/\nu)^{2/3}$ and $T_4 \triangleq (8AD/\nu C^2 \bgd^2)^{2/3}$, for $t \geq T_5 \triangleq \max \left\lbrace T_1, T_2, T_3, T_4 \right\rbrace$, we have
  \begin{align}
      \expect_{\family{}{1:t}{\noise}} \normtwo{\avgeffectgrad{t} - \avglocalgrad{\avgparam{t}}}^2
      &\leq (1+\nu)^3 C^2 K^2 + \nu C^2 \bgd^2/8 + \nu C^2 \bgd^2/8 \\
      &\leq (1+5\nu) C^2 K^2,
  \end{align}
  using the inequality $\nu \leq 1/2$ to show that $(1+\nu)^3 \leq 1+3\nu + 3\nu^2 + \nu^3 \leq 1+3\nu+3\nu/2 + \nu/4$. If we now average this quantity over $t$ from $1$ to $T$, assuming $T \geq T_5$, we can separate the sum from $1$ to $T_5$, and the sum from $T_5+1$ to $T$. This yields
  \begin{equation}
      \frac{1}{T} \sum_{t \in T} \expect_{\family{}{1:t}{\noise}} \normtwo{\avgeffectgrad{t} - \avglocalgrad{\avgparam{t}}}^2
      \leq \frac{T_5 E}{T} +
      \frac{T-T_5}{T} (1+5\nu) C^2 \bgd^2,
  \end{equation}
  where $E = (1+\bar \alpha)^2 (1+12\bar\alpha) C^2 K^2 + AD + B \sigma$. Now consider $T_6 \triangleq T_5 E / \nu C^2 \bgd^2$. For $T \geq T_6$, we then have
  \begin{equation}
      \frac{1}{T} \sum_{t \in T} \expect_{\family{}{1:t}{\noise}} \normtwo{\avgeffectgrad{t} - \avglocalgrad{\avgparam{t}}}^2
      \leq (1+6\nu) C^2 \bgd^2.
  \end{equation}
  Plugging this into Equation (\ref{eq:guanyu_first_bound}), and using $1/(1-\nu) \leq 1+2\nu$ for $0 < \nu \leq 1/2$ then yields, for $T \geq T_6$,
  \begin{equation}
      \expect \normtwo{\avglocalgrad{\avgparam{*}}}^2 \leq \frac{4 \loss_{max}}{\eta T} + (1+\nu)(1+2\nu)(1+6\nu) C^2 \bgd^2.
  \end{equation}
  Now note that $(1+\nu)(1+2\nu) \leq 1+4\nu$ for $\nu \leq 1/2$. Now consider $T_7 \triangleq \loss_{max}/(\nu^2 \eta C^2 \bgd^2)$.
  Then for $T \geq T_8 \triangleq \max \left\lbrace T_6, T_7 \right\rbrace$, we have the guarantee
  \begin{equation}
      \expect \normtwo{\avglocalgrad{\avgparam{*}}}^2 \leq \left((1+4 \nu)(1+6 \nu) + \nu^2)\right) C^2 \bgd^2 \leq (1+6\nu)^2 C^2 \bgd^2.
  \end{equation}
  Now recall that $\nu = \delta/6$, and consider $T_9 \triangleq D \tau^2 /24 \delta^2$ and\footnote{Here, $\tau \approx 6.2832$ is the ratio of the circumference of the circle by its radius.} $T = T_\learn{}(\delta) \triangleq \max \left\lbrace T_8, T_9 \right\rbrace$. Note that we have
  \begin{align}
      \expect \Delta_2\left( \family{}{*}{\param}\right)^2
      &= \frac{1}{T} \sum_{t \in [T]} \expect \Delta_2\left( \family{}{t}{\param}\right)^2
      \leq \frac{1}{T} \sum_{t \in [T]} \frac{D}{t^2} \\
      &\leq \frac{D}{T} \sum_{t=1}^\infty \frac{1}{t^2} = \frac{D}{T} \frac{\tau^2}{24} \leq \delta^2.
  \end{align}
  Combining it all yields
    \begin{equation}
        \expect \Delta_2\left( \family{}{*}{\param}\right)^2 \leq \delta^2 \quad \text{and}\quad \expect \normtwo{ \avglocalgrad{\avgparam{*}}}^2 \leq (1+\delta)^2 C^2 \bgd^2,
    \end{equation}
    which corresponds to saying that \learn{} solves collaborative learning.
\end{proof}

\subsection{Reduction from averaging agreement to collaborative learning}

\begin{proof}[Proof of Theorem \ref{th:inverse_redcution}]
Without loss of generality, assume $0 < \delta \leq 1$. Let $\family{}{}{x} \in \mathbb R^{d \cdot h}$ be a family of vectors.
For any honest node $j \in [h]$, consider the losses defined by $\loss^{(j)}(\param) \triangleq \frac{1}{2} \normtwo{\param - \indexvar{j}{}{x}}^2$. Note that we thus have $\localgrad{j}{\param} = \param-\indexvar{j}{}{x}$.
As a result,
\begin{equation}
    \normtwo{\localgrad{j} {\param} - \localgrad{k} {\param}} = \normtwo{\indexvar{j}{}{x} - \indexvar{k}{}{x}} \leq \Delta_2(\family{}{}{x}),
\end{equation}
which corresponds to saying that local losses satisfy the definition of collaborative learning (Definition \ref{def:collab_learning}) with $K = \Delta_2(\family{}{}{x})$.
Now consider a Byzantine-resilient $C$-collaborative learning algorithm \learn{}. For any $\aggrparameter \in \mathbb N$, we run \learn{} with parameter $\delta \triangleq \min \left\{ 1, \Delta_2(\family{}{}{x})/2^\aggrparameter \right\}$, which outputs $\family{}{\aggrparameter}{x} \triangleq \family{}{*}{\param}$.

We then have the guarantee $\expect \Delta_2(\family{}{\aggrparameter}{x})^2 = \expect \Delta_2(\family{}{*}{\param})^2 \leq \delta^2 \leq \Delta_2(\family{}{}{x})^2 /4^\aggrparameter$, which corresponds to asymptotic agreement. Moreover, we notice that
\begin{equation}
    \avglocalgrad{\avgparam{*}} = \frac{1}{h} \sum_{j \in [h]} (\avgparam{*} - \indexvar{j}{}{x}) = \avgparam{*} - \bar x = \bar x_\aggrparameter - \bar x.
\end{equation}
The second collaborative learning guarantee of algorithm \learn{} (Equation (\ref{eq:constructive_learning_convergence}) in the paper)
then yields
\begin{equation}
    \expect \normtwo{ \bar x_\aggrparameter - \bar x }^2
    = \expect \normtwo{ \avglocalgrad{\avgparam{*}} }^2
    \leq (1+\delta)^2 C^2 \bgd^2
    = \left((1+\delta) C\right)^2 \Delta_2(\family{}{}{x})^2.
\end{equation}
This shows $(1+\delta)C$-averaging, and concludes the proof.
\end{proof}

\section{Efficient i.i.d. algorithm}
\label{sec:identical_losses}
 From a practical viewpoint, a disadvantage of Algorithm \ref{alg:learn} is that it requires a large number of communication rounds due to the use of two instances of the averaging algorithm \aggr{}, one of which requires a stronger agreement (which induces more communication rounds) as the iteration number $t$ grows. This is because when the data distributions of nodes vary from each other, data heterogeneity enables the Byzantines to bias the models in their favor and induce model drift more easily.

 In this section, we show that we can perform better in the homogeneous (i.i.d.) setting, where the ability of Byzantines is more limited. More specifically, we present a simpler algorithm (Algorithm \ref{alg:hom_learn}) that uses the averaging algorithm \aggr{} only once at each iteration with a fixed parameter ($N = 1$). This results in much lower communication/computation overhead compared to Algorithm \ref{alg:learn}. We evaluate the throughput overhead of our two algorithms compared to the non-robust vanilla implantation in Section \ref{sec:eval}, which, not surprisingly, shows the superiority of the i.i.d. algorithm.

 Note that in the homogeneous setting, all local losses are equal, i.e., $\localloss{j}{\cdot} = \loss(\cdot)$ for all honest nodes $j \in [h]$.

\begin{algorithm}[H]
\label{alg:hom_learn}
 \caption{\hlearn{} execution on an honest node.}
    \SetKwFunction{GradientOracle}{GradientOracle}
    \SetKwFunction{QueryGradients}{QueryGradients}
    \SetKwFunction{QueryParameters}{QueryParameters}
    \SetKwFunction{Broadcast}{Broadcast}
    \KwData{Global loss gradient oracle}
    \KwResult{Model parameters $\params{}{t}$}
    \BlankLine
    \BlankLine
    Initialize local parameters $\params{}{1}$ using a fixed seed $s$\;
    Fix learning rate $\eta \leq \frac{1}{2\lipschitz}$\;
    Fix number of rounds $T \triangleq T_\hlearn{}(\delta)$\;
    \For{$t \leftarrow 1, \ldots, T$}{
        $\grads{}{t} \leftarrow \GradientOracle(\params{}{t})$\;
        $\params{}{t+1/2} \leftarrow \params{}{t} - \eta \grads{}{t}$\;
        $\params{}{t+1} \leftarrow \aggr{}_1 \circ \byz{}_{t,\param} \left(\family{}{t+1/2}{\param}  \right)$  \tcp{Vulnerable to Byzantine attacks}
    }
   Draw $* \sim \mathcal U([T])$\;
    Return $\param_{*}$\;
\end{algorithm}

\subsection{Proof of Theorem \ref{th:homo}}
\noindent Before proving our theorem, we prove some preliminary lemmas.

\begin{lemma}
\label{lemma:variance}
  Under Assumption \ref{ass:variance}, we have
  \begin{equation}
    \mathbb E_{\family{}{t}{\noise}|\family{}{t}{\param}} \Delta_2 (\family{}{t}{\noise})^2 \leq 4 \sigma_t^2 h,
  \end{equation}
  where $\noises{j}{t}$ is the gradient estimation error of node $j$ at round $t$ defined in (\ref{eq:error_def}).
\end{lemma}

\begin{proof}

  \begin{align}
      \expect_{\family{}{t}{\noise}|\family{}{t}{\param}} \Delta_2 (\family{}{t}{\noise})^2 &= \expect_{\family{}{t}{\noise}|\family{}{t}{\param}} { \max_{j,k \in [h]} \normtwo{\noises{j}{t} - \noises{k}{t}}^2} \\
      &\leq  4 { \expect_{\family{}{t}{\noise}|\family{}{t}{\param}} \max_{j \in [h]} \normtwo{\noises{j}{t}}^2},
  \end{align}
  where we used Lemma \ref{lemma:diameter_of_small_vectors}.
  We now use the fact that the maximum over nodes $j \in [h]$ is smaller than the sum over nodes $j \in [h]$, yielding
  \begin{align}
      \expect_{\family{}{t}{\noise}|\family{}{t}{\param}} \Delta_2 (\family{}{t}{\noise})^2 &\leq 4 { \expect_{\family{}{t}{\noise}|\family{}{t}{\param}} \sum_{j \in [h]} \normtwo{\noises{j}{t}}^2}
      = 4 {\sum_{j \in [h]}  \expect_{\family{}{t}{\noise}|\family{}{t}{\param}} \normtwo{\noises{j}{t}}^2} \\
      &\leq 4 {\sum_{j \in [h]} \sigma_t^2} = 4\sigma_t^2 h,
  \end{align}
  where the last inequality uses Assumption \ref{ass:variance}.
\end{proof}

\begin{lemma}
\label{lemma:gradient_diameter}
Under assumptions (\ref{ass:lipschitz},\ref{ass:variance}), the expected $\ell_2$ diameter between honest gradient estimations is upper-bounded as follows
\begin{equation}
  \expect_{\family{}{t}{\noise}|\family{}{t}{\param}} \Delta_2 \left(\family{}{t}{\grad}\right)^2 \leq 2 \lipschitz^2 \Delta_2 \left( \family{}{t}{\param} \right)^2 + 8\sigma_t^2h.
\end{equation}
\end{lemma}

\begin{proof}
  By Lemma \ref{lemma:diameter_square_balanced_inequality}, we know that
  \begin{equation}
    \Delta_2 \left(\family{}{t}{\grad}\right)^2 \leq 2 \Delta_2 \left(\localgrad{}{\family{}{t}{\param}}\right)^2 + 2 \Delta_2 \left( \family{}{t}{\noise} \right)^2.
  \end{equation}
  Assumption \ref{ass:lipschitz} then guarantees that
  \begin{align}
    \Delta_2 \left(\localgrad{}{\family{}{t}{\param}}\right)^2
    &= \max_{j,k \in [h]} \normtwo{\localgrad{}{\params{j}{t}} - \localgrad{}{\params{k}{t}}}^2 \\
    &\leq \max_{j,k \in [h]} \lipschitz^2\, \normtwo{\params{j}{t} - \params{k}{t}}^2 \\
    &= \lipschitz^2 \max_{j,k \in [h]} \normtwo{\params{j}{t} - \params{k}{t}}^2
    = \lipschitz^2 \Delta_2 \left( \family{}{t}{\param} \right)^2.
  \end{align}
 Combining this with the previous lemmas completes the proof.
\end{proof}

\begin{lemma}
\label{lemma:diameter_params1/2}
Under assumptions (\ref{ass:lipschitz}, \ref{ass:variance}), we have
\begin{equation}
  \expect_{\family{}{t}{\noise}|\family{}{t}{\param}} \Delta_2 \left( \family{}{t+1}{\param} \right)^2 \leq \left( \frac{1}{2}+\lipschitz^2 \eta^2 \right) \Delta_2 \left( \family{}{t}{\param} \right)^2 + 4 \sigma_t^2 \eta^2 h.
\end{equation}
\end{lemma}

\begin{proof}
  We first bound the diameter of $\params{}{t+1/2}$, using Lemma \ref{lemma:diameter_square_balanced_inequality} and the bound of Lemma \ref{lemma:gradient_diameter}. This yields
  \begin{align}
    \expect_{\family{}{t}{\noise}|\family{}{t}{\param}} \Delta_2 \left( \params{}{t+1/2} \right)^2 &= \expect_{\family{}{t}{\noise}|\family{}{t}{\param}} \Delta_2 \left( \family{}{t}{\param} - \eta \family{}{t}{\grad} \right)^2
    \leq 2 \Delta_2 \left( \family{}{t}{\param} \right)^2 + 2\Delta_2 \left( \eta \family{}{t}{\grad} \right)^2 \\
    &= 2\Delta_2 \left( \family{}{t}{\param} \right)^2 + 2\eta^2 \Delta_2 \left( \family{}{t}{\grad} \right)^2
    \leq \left( 2 + 4\lipschitz^2 \eta^2 \right) \Delta_2 \left( \family{}{t}{\param} \right)^2 + 16 \eta^2 \sigma_t^2 h.
  \end{align}
  We now apply the asymptotic agreement guarantee of \aggr{}, which yields
  \begin{align}
    \expect_{\family{}{t}{\noise}|\family{}{t}{\param}} \Delta_2 \left( \family{}{t+1}{\param} \right)^2 &\leq \frac{1}{4} \expect_{\family{}{t}{\noise}|\family{}{t}{\param}} \Delta_2 \left( \family{}{t+1/2}{\param} \right)^2 \\
    &\leq \frac{1}{4} \left( \left( 2 + 4\lipschitz^2 \eta^2 \right) \Delta_2 \left( \family{}{t}{\param} \right)^2 + 16 \eta^2 \sigma_t^2 h \right) \\
    &\leq \left(\frac{1}{2}+ \lipschitz^2 \eta^2\right) \Delta_2 \left( \family{}{t}{\param} \right)^2 + 4\eta^2 \sigma_t^2 h,
  \end{align}
  which is the lemma.
\end{proof}

\begin{lemma}
\label{lemma:parameter_diameter}
Under assumptions (\ref{ass:lipschitz}, \ref{ass:variance}), the diameter of the parameters is upper-bounded, i.e., there exists a constant $D$ such that for any $t\geq1$ we have
\begin{equation}
  \expect_{\noises{}{1:t}}  \Delta_2 \left( \family{}{t}{\param} \right)^2 \leq D.
\end{equation}

Moreover, for any $\varepsilon>0$, there exists an iteration $T^*(\varepsilon)$, such that for all $t\geq T^*(\varepsilon)$, we have
\begin{equation}
  \expect_{\noises{}{1:t}}  \Delta_2 \left( \family{}{t}{\param} \right)^2 \leq \varepsilon.
\end{equation}

\end{lemma}

\begin{proof}
  We know that   $\eta \leq 1/2\lipschitz$. Denoting $u_t \triangleq \expect \Delta_2 \left( \family{}{t}{\param} \right)^2$, by Lemma \ref{lemma:diameter_params1/2}, we have
  \begin{equation}
  \label{eqn:param_bound}
    u_{t+1} \leq \frac{3}{4} u_t + 4 \sigma_t^2 \eta^2 h.
  \end{equation}
  By induction, we observe that, for all $t \geq 1$,
  \begin{equation}
  \label{eq:diam_bound}
    u_{t+1} \leq 4\eta^2 h \sum_{\tau = 0}^{t-1}   \left(\frac{3}{4}\right)^{\tau} \sigma^2_{t-\tau}.
  \end{equation}
  Now recall that $\sigma_t$ is decreasing (see Section \ref{sec:assumptions}), thus, for all $t\geq 1$, we know $\sigma_t\leq\sigma_1$. Therefore,
  \begin{equation}
    u_{t+1} \leq  4\eta^2 h \sigma_1^2 \sum_{\tau = 0}^{t-1} \left(\frac{3}{4}\right)^{\tau} \leq  4\eta^2 h \sigma_1^2 \sum_{\tau = 0}^{\infty} \left(\frac{3}{4}\right)^{\tau} =   16\eta^2 h \sigma_1^2 \triangleq D.
  \end{equation}
  For the second part of the lemma, recall also from Section \ref{sec:assumptions} that we regularly increase the batch size. Thus, there is an iteration $T_1(\varepsilon)$ such that for all $t\geq T_1(\varepsilon)$ we have $\sigma_t^2 \leq \frac{\varepsilon}{32\eta^2h}$.
  By (\ref{eq:diam_bound}), we then have
  \begin{align}
     u_{t+1} &\leq 4\eta^2 h \sum_{\tau = 0}^{t-T_1(\varepsilon)}   \left(\frac{3}{4}\right)^{\tau} \sigma^2_{t-\tau} + 4\eta^2 h \sum_{\tau = t-T_1(\varepsilon)+1}^{t-1}   \left(\frac{3}{4}\right)^{\tau} \sigma^2_{t-\tau} \\
     &\leq \frac{4\eta^2h\varepsilon}{32\eta^2h} \sum_{\tau = 0}^{t-T_1(\varepsilon)}   \left(\frac{3}{4}\right)^{\tau}+4\eta^2h \sigma_1^2 \sum_{\tau = t-T_1(\varepsilon)+1}^{t-1}   \left(\frac{3}{4}\right)^{\tau} \\
     &\leq \frac{\varepsilon}{8} \sum_{\tau = 0}^{\infty}   \left(\frac{3}{4}\right)^{\tau} + 4\eta^2h \sigma_1^2 \left(\frac{3}{4}\right)^{t-T_1(\varepsilon)+1} \sum_{\tau =0}^{\infty}   \left(\frac{3}{4}\right)^{\tau} \\
     & \leq \frac{\varepsilon}{2} + 16 \eta^2h \sigma_1^2 \left(\frac{3}{4}\right)^{t-T_1(\varepsilon)+1}.
  \end{align}
  Defining $T^*(\varepsilon) \triangleq T_1(\varepsilon) -1 + \frac{\ln(32\eta^2h\sigma_1^2/\varepsilon)}{\ln(4/3)}$, we then have $u_{t+1} \leq \varepsilon$ for $t\geq T^*(\varepsilon)$, which is what we wanted.
\end{proof}

\begin{lemma}
\label{lemma:bounded_grad}
  Under assumptions (\ref{ass:lipschitz}, \ref{ass:variance}), there exist constants $A$ and $B$, such that for all $t\geq1$, we have
  \begin{equation}
    \expect_{\family{}{t}{\noise}|\family{}{t}{\param}} \normtwo{\avgeffectgrad{t} - \localgrad{}{\avgparam{t}}}^2 \leq A \sigma_t^2+B \Delta_2 (\family{}{t}{\param})^2,
  \end{equation}
  where $\avgeffectgrad{t}$ is the average of the effective gradients of the nodes defined in (\ref{eq:avg_eff_grad}).
\end{lemma}

\begin{proof}
  Note that
  \begin{align}
   \avgparam{t+1} - \avgparam{t} &= (\avgparam{t+1} - \avgparam{t+1/2}) + (\avgparam{t+1/2}- \avgparam{t}) \\
   &= (\avgparam{t+1} - \avgparam{t+1/2}) - \eta \avggrad{t}.
  \end{align}
  As a result, $\avgeffectgrad{t} = \avggrad{t} - \frac{1}{\eta} (\avgparam{t+1} - \avgparam{t+1/2})$. Moreover, we have
  \begin{align}
    \avggrad{t} &= \frac{1}{h} \sum_{j \in [h]} \localgrad{}{\params{j}{t}} + \frac{1}{h} \sum_{j \in [h]} \noises{j}{t} \\
    &= \localgrad{}{\avgparam{t}} + \frac{1}{h} \sum_{j \in [h]} \left(\localgrad{}{\params{j}{t}} - \localgrad{}{\avgparam{t}} \right) + \frac{1}{h} \sum_{j \in [h]} \noises{j}{t}.
  \end{align}
  This then yields:
\begin{align}
  \avgeffectgrad{t} - \localgrad{}{\avgparam{t}} &= \frac{1}{h} \sum_{j \in [h]} \left( \localgrad{}{\params{j}{t}} - \localgrad{}{\avgparam{t}} \right) + \frac{1}{h} \sum_{j \in [h]} \noises{j}{t} + \frac{\avgparam{t+1/2} - \avgparam{t+1}}{\eta}.
\end{align}
Taking the $\ell_2$ norm on both sides and invoking Lemma \ref{lemma:diameter_square_balanced_inequality} then implies that

\begin{align}
    \normtwo{\avgeffectgrad{t} - \localgrad{}{\avgparam{t}}}^2 &\leq 3\normtwo{\frac{1}{h} \sum_{j \in [h]} \left( \localgrad{}{\params{j}{t}} - \localgrad{}{\avgparam{t}} \right)}^2 \nonumber \\&+ 3\normtwo{\frac{1}{h} \sum_{j \in [h]} \noises{j}{t}}^2 + 3\normtwo{\frac{\avgparam{t+1/2} - \avgparam{t+1}}{\eta}}^2.
\end{align}
We now note that the expectation of each term can be bounded. Indeed,
\begin{align}
    &\normtwo{\frac{1}{h} \sum_{j \in [h]} \left( \localgrad{}{\params{j}{t}} - \localgrad{}{\avgparam{t}} \right)} \leq \frac{1}{h} \sum_{j \in [h]} \normtwo{\localgrad{}{\params{j}{t}} - \localgrad{}{\avgparam{t}}} \\
    &\qquad \qquad \qquad \leq \frac{1}{h} \sum_{j \in [h]} \lipschitz \normtwo{\params{j}{t} - \avgparam{t}}
    \leq \frac{1}{h} \sum_{j \in [h]} \lipschitz \Delta_2(\family{}{t}{\param}) = \lipschitz \Delta_2(\family{}{t}{\param}).
\end{align}
Moreover,
\begin{align}
    \expect \normtwo{\frac{1}{h} \sum_{j \in [h]} \noises{j}{t}}^2
    &= \expect \frac{1}{h^2} \sum_{j,k \in [h]} \noises{j}{t} \cdot \noises{k}{t}
    = \frac{1}{h^2} \sum_{j,k \in [h]} \expect \noises{j}{t} \cdot \noises{k}{t} \\
    &= \frac{1}{h^2} \sum_{j \in [h]} \expect \normtwo{\noises{j}{t}}^2
    \leq \frac{1}{h^2} \sum_{j \in [h]} \sigma_t^2 = \frac{\sigma_t^2}{h},
\end{align}
using the fact that the noises are independent to move from one line to the other.
For the last term, we use the $C$-averaging guarantee of \aggr{}, yielding
\begin{equation}
    \normtwo{\avgparam{t+1} - \avgparam{t+1/2}}^2 \leq C^2 \Delta_2 (\family{}{t+1/2}{\param})^2.
\end{equation}
To bound the right-hand side, note that we have $\family{}{t+1/2}{\param} = \family{}{t}{\param} - \eta \family{}{t}{\grad} = \family{}{t}{\param} - \eta \overrightarrow{\nabla \loss} (\family{}{t}{\param}) - \eta \family{}{t}{\noise}$, where $\overrightarrow{\nabla \loss} (\family{}{t}{\param}) = \left( \localgrad{}{\indexvar{1}{t}{\param}}, \ldots, \localgrad{}{\indexvar{h}{t}{\param}} \right)$. Lemma \ref{lemma:diameter_square_balanced_inequality} then implies
\begin{equation}
    \normtwo{\frac{\avgparam{t+1} - \avgparam{t+1/2}}{\eta}}^2
    \leq \frac{3C^2}{\eta^2} \Delta_2 (\family{}{t}{\param})^2
    + 3 C^2 \lipschitz^2 \Delta_2 (\family{}{t}{\param})^2
    + 3 C^2 \Delta_2  (\family{}{t}{\noise})^2.
\end{equation}
Combining it all, applying Lemma \ref{lemma:variance} and defining $B \triangleq \frac{9C^2}{\eta^2} +  3\lipschitz^2 + 9 C^2 \lipschitz^2$, then yields
\begin{align}
    \expect_{\family{}{t}{\noise} | \family{}{t}{\param}} \normtwo{\avgeffectgrad{t} - \localgrad{}{\avgparam{t}}}^2
    &\leq \frac{3\sigma_t^2}{h} + \left( \frac{9C^2}{\eta^2} +  3\lipschitz^2 + 9C^2 \lipschitz^2 \right) \Delta_2 (\family{}{t}{\param})^2 + 9 C^2 \expect_{\family{}{t}{\noise} | \family{}{t}{\param}} \Delta_2 (\family{}{t}{\noise})^2 \\
    &\leq \left( \frac{3}{h}+36 C^2h \right) \frac{\sigma_t^2}{h} +B \Delta_2 (\family{}{t}{\param})^2.
\end{align}
Defining $A \triangleq \frac{3}{h}+36 C^2h$ then yields the desired result.
\end{proof}

We now proceed with the proof of our theorem.

\begin{proof}[Proof of Theorem \ref{th:homo}]
  A direct consequence of Lipschitz continuity of the gradient of the loss function (Assumption \ref{ass:lipschitz}) is that for all $\phi, \psi \in \setr^d$, we have
  \begin{equation}
      \loss(\psi) \leq \loss(\phi) + (\psi-\phi) \cdot \localgrad{}{\phi} + \frac{\lipschitz}{2} \normtwo{\psi-\phi}^2.
  \end{equation}
  Therefore, by the definition of the effective gradient, we have
  \begin{equation}
      \loss(\avgparam{t+1})
      \leq \loss(\avgparam{t}) - \eta \avgeffectgrad{t} \cdot \localgrad{}{\avgparam{t}} + \frac{\lipschitz \eta^2}{2} \normtwo{\avgeffectgrad{t}}^2.
  \end{equation}
  The fact that $\eta \leq 1/\lipschitz$ then implies
  \begin{align}
      \loss(\avgparam{t+1})
      &\leq \loss(\avgparam{t}) - \eta \avgeffectgrad{t} \cdot \localgrad{}{\avgparam{t}} + \frac{\eta}{2} \normtwo{\avgeffectgrad{t}}^2 \\
      &= \loss(\avgparam{t}) - \frac{\eta}{2}\normtwo{\localgrad{}{\avgparam{t}}}^2+\frac{\eta}{2} \left(\normtwo{\avgeffectgrad{t}}^2- 2\avgeffectgrad{t} \cdot \localgrad{}{\avgparam{t}}+\normtwo{\localgrad{}{\avgparam{t}}}^2 \right)\\
      &= \loss(\avgparam{t}) - \frac{\eta}{2} \normtwo{\localgrad{}{\avgparam{t}}}^2+\frac{\eta}{2} \normtwo{\avgeffectgrad{t}-\localgrad{}{\avgparam{t}}}^2.
  \end{align}
  By rearranging the terms, we then have
  \begin{equation}
      \normtwo{\localgrad{}{\avgparam{t}}}^2 \leq \frac{2}{\eta} \left( \loss(\avgparam{t}) - \loss(\avgparam{t+1})\right)+\normtwo{\avgeffectgrad{t}-\localgrad{}{\avgparam{t}}}^2.
  \end{equation}
  Now taking the expectation over all of the stochastic noises and averaging over $t$, yields
  \begin{align}
      \frac{1}{T}  \sum_{t \in [T]}  \expect_{\family{}{1:T}{\noise}} \normtwo{\localgrad{}{\avgparam{t}}}^2  &\leq \frac{2}{\eta T}\left(\loss(\avgparam{1})-\loss(\avgparam{T+1})\right)  + \frac{1}{T} \sum_{t \in [T]} \expect_{\family{}{1:T}{\noise}} \normtwo{\avgeffectgrad{t} - \localgrad{}{\avgparam{t}}}^2 \\ &\leq \frac{2 \loss_{max}}{\eta T} + \frac{A}{T} \sum_{t \in [T]} \sigma_t^2+  \frac{B}{T} \sum_{t \in [T]} \expect_{\family{}{1:T}{\noise}} \Delta_2 (\family{}{t}{\param})^2,
  \end{align}
  where in the last inequality we used Lemma \ref{lemma:bounded_grad}. The regular increase of the batch size (Section \ref{sec:assumptions}) then implies the existence of an iteration $T_1$, after which we have $\sigma_t^2\leq \frac{\delta^2}{8A}$. Therefore, for $T\geq T_2\triangleq \max \left\{\frac{8 \loss_{max}}{\eta \delta^2},T_1\right\}$, we have
  \begin{align}
    \frac{1}{T}  \sum_{t \in [T]}  \expect_{\family{}{1:T}{\noise}} \normtwo{\localgrad{}{\avgparam{t}}}^2  &\leq \frac{\delta^2}{4} + \frac{A}{T} \left(T_2 \sigma_1^2 + (T-T_2)\frac{\delta^2}{8A} \right) + \frac{B}{T} \sum_{t \in [T]} \expect_{\family{}{1:T}{\noise}} \Delta_2 (\family{}{t}{\param})^2\\
    &\leq \frac{3\delta^2}{8} + \frac{T_2}{T} A\sigma_1^2 + \frac{B}{T} \sum_{t \in [T]} \expect_{\family{}{1:T}{\noise}} \Delta_2 (\family{}{t}{\param})^2.
  \end{align}
  For $T \geq T_3 \triangleq \frac{8T_2A\sigma_1^2}{\delta^2}$, we then have
  \begin{equation}
  \label{eq:grad_bound}
     \frac{1}{T}  \sum_{t \in [T]}  \expect_{\family{}{1:T}{\noise}} \normtwo{\localgrad{}{\avgparam{t}}}^2 \leq \frac{\delta^2}{2} + \frac{B}{T} \sum_{t \in [T]} \expect_{\family{}{1:T}{\noise}} \Delta_2 (\family{}{t}{\param})^2.
  \end{equation}
  We now invoke Lemma \ref{lemma:parameter_diameter} with $\varepsilon = \frac{1}{2}\min \left\{\delta^2, \delta^2/2B\right\}$. For $T\geq T^*(\varepsilon)$ we then have
  \begin{align}
    \frac{1}{T} \sum_{t \in [T]} \expect_{\family{}{1:T}{\noise}} \Delta_2 (\family{}{t}{\param})^2 &=  \frac{1}{T} \sum_{t = 1}^{T^*(\varepsilon)} \expect_{\family{}{1:T}{\noise}} \Delta_2 (\family{}{t}{\param})^2 + \frac{1}{T} \sum_{t=T^*(\varepsilon)+1}^T \expect_{\family{}{1:T}{\noise}} \Delta_2 (\family{}{t}{\param})^2 \\
    &\leq\frac{T^*(\varepsilon)}{T}D+\frac{T-T^*(\varepsilon)}{T}\varepsilon \leq \frac{T^*(\varepsilon)}{T}D + \varepsilon.
  \end{align}
  Thus, for $T\geq T_4 \triangleq \frac{DT^*(\varepsilon)}{\varepsilon}$ we have
  \begin{equation}
    \expect\Delta_2 (\family{}{*}{\param})^2=\frac{1}{T} \sum_{t \in [T]} \expect_{\family{}{1:T}{\noise}} \Delta_2 (\family{}{t}{\param})^2 \leq 2\varepsilon = \min \left\{\delta^2, \delta^2/2B\right\}.
  \end{equation}
  Combining this with (\ref{eq:grad_bound}), for $T = T_\hlearn{}(\delta) \triangleq \max \left\{ T_3,T_4\right\}$ we then have
  \begin{equation}
    \expect\Delta_2 (\family{}{*}{\param})^2 \leq \delta^2 \quad \text{and} \quad \expect \normtwo{\localgrad{}{\avgparam{*}}}^2 \leq \delta^2,
  \end{equation}
  which is the desired result.

\end{proof}

\section{\brute{} Algorithm}
\subsection{Correctness proof of \brute{}}
We first note a few important properties of \brute{}.

\begin{lemma}
\label{lemma:brute_diameter}
  The $\ell_2$ diameter of the \brute{} subfamily is upper-bounded by that of the honest vectors. In other words, for any Byzantine attack $\byzfamily{}{}$, denoting $\family{}{}{z} \triangleq \byzfamily{}{}(\family{}{}{x})$, we have
  \begin{align}
    \Delta_2 \left( \family{\brute{}}{}{z} \right)
    \leq \Delta_2 \left(\family{}{}{x}\right)
  \end{align}
\end{lemma}

\begin{proof}
  Since $\byzfamily{}{}$ selects $q$ vectors, out of which at most $f$ are Byzantine vectors, we know that there exists a subset $H \subset [q]$ of cardinal at least $q-f$ that only contains honest vectors. But then, we have
  \begin{equation}
    \Delta_2 \left( \family{\brute{}}{}{z} \right)
    = \min_{\underset{\card{S} = q-f}{S \subset [q]}} \Delta_2 \left(\family{S}{}{z}\right)
    \leq \Delta_2 \left(\family{H}{}{z}\right) \leq \Delta_2(\family{}{}{x}),
  \end{equation}
  which is the lemma.
\end{proof}

\begin{lemma}
\label{lemma:brute_contraction}
  Under Assumption \ref{ass:query_brute}, \brute{} guarantees Byzantine asymptotic agreement. In other words, for any input $\aggrparameter \in \setn$ and any family $\family{}{}{x} \in \setr^{d \cdot h}$, denoting $\family{}{\aggrparameter}{x} \triangleq \overrightarrow{\brute{}}_\aggrparameter \circ \byzfamily{}{\aggrparameter} (\family{}{}{x})$,
  \begin{equation}
    \Delta_2 (\family{}{\aggrparameter}{x}) \leq \frac{\Delta_2 (\family{}{}{x})}{2^\aggrparameter}.
  \end{equation}
\end{lemma}

\begin{proof}
  Denote $\family{1}{}{z} \triangleq \byzfamily{1}{}(\family{}{}{x})$ and $\family{2}{}{z} \triangleq \byzfamily{2}{}(\family{}{}{x})$ the results of the two Byzantine attacks, $S_1 = S_{\brute{}}(\family{1}{}{z})$ and $S_2 = S_{\brute{}}(\family{2}{}{z})$ the subsets selected by \brute{} in the two cases.

  Moreover, we write $S_1 = H_1 \cup F_1$ and $S_2 = H_2 \cup F_2$, where $H_1$ and $H_2$ are subsets of honest vectors within $S_1$ and $S_2$. Without loss of generality, we assume both $H_1$ and $H_2$ to be of cardinal $q-2f$. As a result, we know that there exist injective functions $\sigma_1 : H_1 \rightarrow [h]$ and $\sigma_2 : H_2 \rightarrow [h]$ such that $\indexvar{1,j}{}{z} = \indexvar{\sigma_1(j)}{}{x}$ and $\indexvar{2,k}{}{z} = \indexvar{\sigma_2(k)}{}{x}$, for all $j \in H_1$ and $k \in H_2$.

  Finally, we denote $\indexvar{1}{}{y} \triangleq \brute{} (\family{1}{}{z})$ and $\indexvar{2}{}{y} \triangleq \brute{} (\family{2}{}{z})$. We then have
  \begin{align}
    &(q-f) \normtwo{\indexvar{1}{}{y} - \indexvar{2}{}{y}}
    = \normtwo{\sum_{j \in S_1} \indexvar{1,j}{}{z} - \sum_{k \in S_2} \indexvar{2,k}{}{z}} \\
    &= \normtwo{\sum_{j \in F_1} \indexvar{1,j}{}{z} - \sum_{k \in F_2} \indexvar{2,k}{}{z} + \sum_{j \in \sigma_1(H_1)} \indexvar{j}{}{x} - \sum_{k \in \sigma_2(H_2)} \indexvar{k}{}{x}} \\
    &\leq \normtwo{\sum_{j \in F_1} \indexvar{1,j}{}{z} - \sum_{k \in F_2} \indexvar{2,k}{}{z}} + \normtwo{\sum_{j \in \sigma_1(H_1)-\sigma_2(H_2)} \indexvar{j}{}{x} - \sum_{k \in \sigma_2(H_2)- \sigma_1(H_1)} \indexvar{k}{}{x}}.
  \end{align}
  Note that $\card{F_1} = \card{S_1-H_1} = f = \card{S_2-H_2} = \card{F_2}$. Moreover,
  \begin{align}
    \card{\sigma_1(H_1) - \sigma_2(H_2)} &= \card{\sigma_1(H_1) \cup \sigma_2(H_2) - \sigma_2(H_2)} \\
    &= \card{\sigma_1(H_1) \cup \sigma_2(H_2)} - \card{\sigma_2(H_2)} \leq \card{[h]} - \card{H_2} = 2f+h-q,
  \end{align}
  and similarly for $\sigma_2(H_2)- \sigma_1(H_1)$. Note that $\card{\sigma_1(H_1)} , \card{\sigma_2(H_2)} \geq q - 2f$. Therefore, Assumption~\ref{ass:query_brute} implies that
  \begin{equation}
    \card{\sigma_1(H_1)} + \card{\sigma_2(H_2)} \geq 2 \left( \frac{1+\varepsilon}{2} h + \frac{5+3\varepsilon}{2} f - 2f \right) > h,
  \end{equation}
  which yields $\sigma_1(H_1) \cap \sigma_2(H_2) \neq \varnothing$. Now let $\gamma$ be an element of the intersection of $\sigma_1(H_1)$ and $\sigma_2(H_2)$, and consider any bijections $\tau_F : F_1 \rightarrow F_2$ and $\tau_H : \sigma_1(H_1)-\sigma_2(H_2) \rightarrow \sigma_2(H_2)- \sigma_1(H_1)$.
  Using triangle inequality and Lemma \ref{lemma:brute_diameter}, for any $j \in F_1$, we then have
  \begin{equation}
    \normtwo{\indexvar{1,j}{}{z}-\indexvar{2,\tau_F(j)}{}{z}} \leq \normtwo{\indexvar{1,j}{}{z}-\indexvar{\gamma}{}{x}} + \normtwo{\indexvar{\gamma}{}{x}-\indexvar{2,\tau_F(j)}{}{z}} \leq 2 \Delta_2(\family{}{}{x}).
  \end{equation}
  Combining it all, we obtain
  \begin{align}
    (q-f) \normtwo{\indexvar{1}{}{y} - \indexvar{2}{}{y}}
    &\leq \sum_{j \in F_1} \normtwo{ \indexvar{1,j}{}{z} - \indexvar{2,\tau_F(j)}{}{z}} + \sum_{j \in \sigma_1(H_1)-\sigma_2(H_2)} \normtwo{\indexvar{j}{}{x} - \indexvar{\tau_H(j)}{}{x}} \\
    &\leq 2f \Delta_2(\family{}{}{x}) + (2f+h-q) \Delta_2(\family{}{}{x}),
  \end{align}
  which implies
  \begin{equation}
    \normtwo{\indexvar{1}{}{y}-\indexvar{2}{}{y}} \leq \frac{4f+h-q}{q-f} \Delta_2(\family{}{}{x}).
  \end{equation}
  We then apply Assumption \ref{ass:query_brute}, which implies that
  \begin{align}
    \frac{4f+h-q}{q-f} &
    \leq \frac{4f+h-\frac{1+\varepsilon}{2}h-\frac{5+3\varepsilon}{2}f}{\frac{1+\varepsilon}{2}h+\frac{5+3\varepsilon}{2}f-f}
    = \frac{(1-\varepsilon)h + 3(1-\varepsilon)f}{(1+\varepsilon) h + 3(1+\varepsilon)f} \\
    &= \frac{1-\varepsilon}{1+\varepsilon} = \frac{1+\varepsilon-2\varepsilon}{1+\varepsilon} = 1- \frac{2\varepsilon}{1+\varepsilon} = 1-\tilde \varepsilon.
  \end{align}
  This shows that $\Delta_2(\family{}{}{y}) \leq (1-\tilde \varepsilon) \Delta_2(\family{}{}{x})$. In other words, one iteration of \brute{} is guaranteed to multiply the $\ell_2$ diameter of honest nodes by at most $(1-\tilde \varepsilon)$.
  It follows that $T_\brute{}(\aggrparameter) = \left\lceil \aggrparameter \ln 2/ \tilde \varepsilon\right \rceil$ iterations will multiply this diameter by at most $(1-\tilde \varepsilon)^{T_\brute{}(\aggrparameter)} \leq \exp\left( \frac{\ln(1-\tilde \varepsilon) \ln 2}{\tilde \varepsilon} \right)^\aggrparameter \leq 2^{-\aggrparameter}$, using the inequality $\ln(1-\tilde \varepsilon) \leq - \tilde \varepsilon$.
\end{proof}

\begin{remark}
Note that we can set $\varepsilon =  \frac{n-6f}{n+2f}$. Thus, in the regime $f \ll n$, we have $\varepsilon \rightarrow 1$, which implies $\tilde \varepsilon \rightarrow 1$. Thus, for any fixed value of $\aggrparameter$, given that our proof showed that $\Delta_2(\family{}{}{y}) \leq (1-\tilde \varepsilon) \Delta_2(\family{}{}{x})$, when $f$ is sufficiently small compared to $n$, \brute{} actually achieves asymptotic agreement in only one communication round.
\end{remark}

\begin{lemma}
\label{lemma:brute_average_guarantee}
  One iteration of \brute{} returns a vector close to the average of the honest vectors. Denoting $\family{}{}{y} \triangleq \overrightarrow{\brute{}} \circ \byzfamily{}{} (\family{}{}{x})$, we have
  \begin{align}
    \normtwo{\bar y - \bar x}
    \leq \frac{(2f+h-q)q + (q-2f)f}{h(q-f)} \Delta_2 \left(\family{}{}{x}\right).
  \end{align}
  In the synchronous case where $q=n=f+h$, the right-hand side becomes $\frac{2f}{h} \Delta_2 \left(\family{}{}{x}\right)$.
\end{lemma}

\begin{proof}
  Let us write $S_{\brute{}} (\family{}{}{z}) = H \cup F$, where $H$ are honest vectors and $F$ are Byzantine vectors. We know that $\card{H} \geq q-2f$ and $\card{H} + \card{F} = q-f$. In fact, without loss of generality, we can assume $\card{H} = q-2f$ (since this is equivalent to labeling honest vectors not in $H$ as Byzantine vectors).

  Let us also denote $\sigma : H \rightarrow [h]$ the injective function that maps honest vectors to the index of their node, and $\bar H = [h] - \sigma(H)$ the unqueried nodes. We have
  \begin{align}
    &\normtwo{y - \bar x}
    = \normtwo{\frac{\card{H} \indexvar{H}{}{\bar z} + \card{F} \indexvar{F}{}{\bar z}}{\card{H} + \card{F}} - \frac{\card{\sigma(H)} \indexvar{\sigma(H)}{}{\bar x} + \card{\bar H} \indexvar{\bar H}{}{\bar x}}{\card{\sigma(H)} + \card{\bar H}} } \\
    &= \normtwo{\frac{\card{H} \indexvar{H}{}{\bar z} + \card{F} \indexvar{F}{}{\bar z}}{\card{H} + \card{F}} - \frac{\card{H} \indexvar{H}{}{\bar z} + \card{\bar H} \indexvar{\bar H}{}{\bar x}}{\card{H} + \card{\bar H}} } \\
    &= \frac{\normtwo{
    \card{H} \left( \card{\bar H} - \card{F} \right) \indexvar{H}{}{\bar z} +
    \card{F} \left( \card{H} + \card{\bar H} \right) \indexvar{F}{}{\bar z} -
    \card{\bar H} \left( \card{H} + \card{F} \right) \indexvar{\bar H}{}{\bar x}
     }}{\left( \card{H} + \card{F} \right) \left( \card{H} + \card{\bar H} \right)} \\
    &= \frac{\normtwo{
     \card{H} \card{\bar H} \left(\indexvar{H}{}{\bar z} - \indexvar{\bar H}{}{\bar x}\right) +
     \card{F} \card{H} \left(\indexvar{H}{}{\bar z} - \indexvar{F}{}{\bar z}\right) +
     \card{\bar H} \card{F} \left(\indexvar{F}{}{\bar z} - \indexvar{\bar H}{}{\bar x}\right)
      }}{\left( \card{H} + \card{F} \right) \left( \card{H} + \card{\bar H} \right)} \\
    &\leq \frac{
    \card{H} \card{\bar H} \normtwo{\indexvar{H}{}{\bar z} - \indexvar{\bar H}{}{\bar x}} +
    \card{F} \card{H} \normtwo{\indexvar{H}{}{\bar z} - \indexvar{F}{}{\bar z}} +
    \card{\bar H} \card{F} \normtwo{\indexvar{F}{}{\bar z} - \indexvar{\bar H}{}{\bar x}}
    }{\left( \card{H} + \card{F} \right) \left( \card{H} + \card{\bar H} \right)}.
  \end{align}
  Now note that
  \begin{align}
    \normtwo{\indexvar{H}{}{\bar z} - \indexvar{\bar H}{}{\bar x}}
    &= \normtwo{\indexvar{\sigma(H)}{}{\bar x} - \indexvar{\bar H}{}{\bar x}} \\
    &= \frac{1}{\card{\sigma(H)} \card{\bar H}} \normtwo{\card{\bar H} \sum_{j \in \sigma(H)} \indexvar{j}{}{x} - \card{H} \sum_{k \in \bar H} \indexvar{k}{}{x}} \\
    &= \frac{1}{\card{\sigma(H)} \card{\bar H}} \normtwo{\sum_{k \in \bar H} \sum_{j \in \sigma(H)} \indexvar{j}{}{x} - \sum_{j \in \sigma(H)} \sum_{k \in \bar H} \indexvar{k}{}{x}} \\
    &= \frac{1}{\card{\sigma(H)} \card{\bar H}} \normtwo{\sum_{j \in \sigma(H)} \sum_{k \in \bar H} \left( \indexvar{j}{}{x} - \indexvar{k}{}{x}\right)} \\
    &\leq \frac{1}{\card{\sigma(H)} \card{\bar H}} \sum_{j \in \sigma(H)} \sum_{k \in \bar H} \normtwo{\indexvar{j}{}{x} - \indexvar{k}{}{x}} \\
    &\leq \frac{1}{\card{\sigma(H)} \card{\bar H}} \sum_{j \in \sigma(H)} \sum_{k \in \bar H} \Delta_2(\family{}{}{x})
    = \Delta_2(\family{}{}{x}).
  \end{align}
  Similarly, we show that $\normtwo{\indexvar{H}{}{\bar z} - \indexvar{F}{}{\bar z}} \leq \Delta_2(\family{\brute{}}{}{z}) \leq \Delta_2(\family{}{}{x})$. Finally, we use the triangle inequality to show that
  \begin{equation}
    \normtwo{\indexvar{F}{}{\bar z} - \indexvar{\bar H}{}{\bar x}} \leq \normtwo{\indexvar{F}{}{\bar z} - \indexvar{H}{}{\bar z}} + \normtwo{\indexvar{H}{}{\bar z} - \indexvar{\bar H}{}{\bar x}} \leq 2 \Delta_2(\family{}{}{x}).
  \end{equation}
  Therefore, we now have
  \begin{align}
    \normtwo{y-\bar x} &\leq \frac{\card{H} \card{\bar H} + \card{F} \card{H} + 2 \card{\bar H} \card{F}}{\left( \card{H} + \card{F} \right) \left( \card{H} + \card{\bar H} \right)} \Delta_2(\family{}{}{x}) \\
    &= \frac{(q-2f) (2f+h-q) + f (q-2f) + 2 (2f+h-q) f}{h(q-f)} \Delta_2(\family{}{}{x}) \\
    &= \frac{(2f+h-q)q + (q-2f)f}{h(q-f)} \Delta_2(\family{}{}{x}),
  \end{align}
  which is the lemma.
\end{proof}

Now we can prove our theorem.
\begin{proof}[Proof of Theorem \ref{th:brute}]
  Lemma \ref{lemma:brute_contraction} already proved asymptotic agreement. Moreover, using Lemma \ref{lemma:brute_average_guarantee}, and denoting $\alpha \triangleq \frac{(2f+h-q)q + (q-2f)f}{h(q-f)}$ and $\family{}{t}{x}$ the vector family obtained after $t$ iterations of \brute{}, we also know that
  \begin{equation}
    \normtwo{\indexvar{}{t+1}{\bar x} - \indexvar{}{t}{\bar x}} \leq \alpha \Delta_2(\family{}{t}{x}) \leq \alpha (1- \tilde \varepsilon)^t \Delta_2(\family{}{0}{x}).
  \end{equation}
  Using triangle inequality then yields, for any number $T_{\brute{}}(N)$ of iterations of \brute{},
  \begin{align}
    \normtwo{\indexvar{}{T_{\brute{}}(N)}{\bar x} - \indexvar{}{0}{\bar x}}
    &\leq \sum_{t=0}^{T_{\brute{}}(N)-1} \normtwo{\indexvar{}{t+1}{\bar x} - \indexvar{}{t}{\bar x}}
    \leq \sum_{t=0}^{T_{\brute{}}(N)-1} \alpha (1-\tilde \varepsilon)^t \Delta_2(\family{}{0}{x}) \\
    &\leq \alpha \Delta_2(\family{}{0}{x}) \sum_{t=0}^{\infty} (1-\tilde \varepsilon)^t
    = \frac{\alpha \Delta_2(\family{}{0}{x})}{\tilde \varepsilon},
  \end{align}
  which is the guarantee of the theorem.
\end{proof}

\subsection{Lower bound on the averaging constant}

We prove here a lower bound on the averaging constant that any algorithm can achieve.
Our proof requires the construction of hard instances. We use the following notation.

\begin{definition}[$\star$ notation]
We denote by $x \star h \triangleq \big( \underbrace{x, \ldots, x}_{h\text{ times}} \big)$ the repetition of a value $x$ $h$ times.
\end{definition}

\begin{lemma}[Quasi-unanimity]
\label{lemma:quasi-unanimity}
For any averaging agreement algorithm \aggr{}, whenever a node $j$ only hears from $q$ nodes, assuming $q-f$ of these nodes act like honest nodes with the same initial value $x$, then \aggr{} must make node $j$ output $x$.
\end{lemma}

\begin{proof}
For any agreeing initial family $\family{}{}{x} = x \star h$, we have $\Delta_2(\family{}{}{x}) = 0$ and $\bar x = x$. Then averaging agreement implies that, for any $\aggrparameter \in \setn$, we output $\family{}{\aggrparameter}{x}$ such that
\begin{equation}
    \Delta_2(\family{}{\aggrparameter}{x}) \leq \frac{\Delta_2(\family{}{}{x})}{2^\aggrparameter} = 0 \quad \text{and} \quad \normtwo{\bar x_\aggrparameter - x} \leq C \Delta_2(\family{}{}{x}) = 0.
\end{equation}
In other words, we must have $\family{}{\aggrparameter}{x} = \family{}{}{x}$.

But then, if node $j$ only hears from $q$ nodes, and if it receives $q-f$ nodes agreeing on a value $x$, then it cannot exclude the possibility that the remaining $f$ nodes come from Byzantine nodes. As a result, node $j$ cannot exclude that the initial family was $\family{}{}{x}$. To satisfy averaging agreement, node $j$ must then output $x$.
\end{proof}

\begin{proof}[Proof of Theorem \ref{th:lower_bound_averaging_constant}]
Consider the vector family defined by
\begin{equation}
    \family{}{}{x} \triangleq \left( 0 \star (q-2f), 1 \star (h+2f-q) \right).
\end{equation}
For any algorithm \aggr{} used by honest nodes, Byzantine nodes can slow down all messages from nodes in $[q-f+1,h]$ to nodes in $[q-f]$.
Thus, the first $q-f$ honest nodes would be making decisions without receiving any input from nodes in $[q-f+1,h]$. Assume now that the Byzantine nodes all act exactly like the first $q-2f$ nodes. Then, all first $q-f$ nodes would see $q-f$ nodes acting like honest nodes with initial vector $0$, and $f$ nodes acting like honest nodes with initial vector $1$. By quasi-unanimity (Lemma \ref{lemma:quasi-unanimity}), the $q-2f$ first nodes must output $0$.

But now, by asymptotic agreement, this implies that any other honest node must output a vector at distance at most $\Delta_2(\family{}{}{x})/2^\aggrparameter = 1/2^\aggrparameter$ of 0. As a result, as $\aggrparameter \rightarrow \infty$ , denoting $\family{}{\aggrparameter}{x}$ the output of \aggr{} for input $\aggrparameter \in \setn$, we must have $\bar x_\aggrparameter \rightarrow 0$.

Since $\Delta_2(\family{}{}{x}) = 1$ and $\bar x = (h+2f-q)/h$, we then have
\begin{equation}
    \lim_{\aggrparameter \rightarrow 0} \absv{ \bar x_\aggrparameter - \bar x} = \absv{0 -\bar x} = \frac{h+2f-q}{h} \geq \frac{h+2f-q}{h} \Delta_2(\family{}{}{x}).
\end{equation}
This shows that \aggr{} cannot achieve better than $\frac{h+2f-q}{h}$-averaging agreement which is equal to $\frac{2f}{h}$, for $q = h$.

Now we show \brute{} indeed achieves this bound up to a multiplicative constant. From Assumption \ref{ass:query_brute}, we can set
$\varepsilon = \frac{n-6f}{n+2f} = \frac{h-5f}{h+3f}$. In the regime $q = h$ and $f \ll h$, we thus have $\varepsilon \rightarrow 1$, which then implies $\tilde \varepsilon \rightarrow 1$.
Now notice that
\begin{align}
    \frac{(2f+h-q)q + (q-2f)f}{h(q-f) \tilde \varepsilon} &= \frac{2f/h}{(1-\frac{f}{q}) \tilde \varepsilon} +  \frac{f}{h} \frac{1-\frac{2f}{q}}{(1-\frac{f}{q}) \tilde \varepsilon}
    = \frac{3f}{h} + o(1).
\end{align}
Yet, for $q=h$, we showed that the lower bound on the averaging constant is $2f/h$. \brute{} is thus asymptotically optimal, up to the multiplicative constant $3/2$.
\end{proof}

\subsection{Note on Byzantine tolerance}

Our \brute{} algorithm tolerates a small fraction of Byzantine nodes, namely $n > 6f$.

\begin{proposition}
Assume $n \leq 6f$. For any parameter $\aggrparameter$, no matter how the number $T_\brute{}(\aggrparameter)$ of iterations is chosen, Byzantines can make \brute{} fail to achieve asymptotic agreement. As a result, \brute{} cannot guarantee averaging agreement for $n \leq 6f$.
\end{proposition}

\begin{proof}
Define $\delta \triangleq \min\left\{ 1, 4 - 4 \cdot 2^{-(N-1)/T_{\brute{}}(N)} \right\}$, and consider the honest vector family
\begin{equation}
\family{}{}{x} \triangleq (-1 \star 2f, 0 \star f, 1 \star 2f).
\end{equation}
For the first $2f$ nodes, Byzantine nodes can block $f$ of the messages of the last $2f$ nodes, and add $f$ values equal to $-2+\delta$.
The first $2f$ honest nodes then observe
\begin{equation}
\family{}{1}{z} \triangleq ((-2+ \delta) \star f, -1 \star 2f, 0 \star f, 1 \star f).
\end{equation}
\brute{} would then remove the largest $f$ inputs, as it then achieves a diameter equal to $2-\delta < 2$. The first $2f$ nodes would then output
\begin{equation}
    \frac{(-2+\delta)f-2f}{4f} = -\frac{4f-\delta f}{4f} = -1 + \frac{\delta}{4}.
\end{equation}
For the middle $f$ nodes, Byzantine nodes can simply allow perfect communication and not intervene. By symmetry, the middle $f$ nodes would output $0$ (note that to be rigorous, we could define \brute{} as taking the average of the outputs over all subsets of inputs of minimal diameter).

For the last $2f$ nodes, Byzantine nodes can block $f$ messages of the first $2f$ nodes, and add $f$ values equal to $2-\delta$. The situation is then symmetric to the case for the first $2f$ nodes, and make the last $2f$ nodes output $1-\frac{\delta}{4}$.

As a result, one iteration of \brute{} multiplies the diameter of the honest nodes by $1-\frac{\delta}{4} \geq 2^{-(N-1)/T_\brute{}(N)}$. Byzantine nodes can use the same strategy in all other iterations. Thus, denoting  $\family{}{}{y} \triangleq \overrightarrow{\brute{}} \circ \byzfamily{}{} (\family{}{}{x})$, we have
\begin{equation}
    \Delta_2(\family{}{}{y}) = \left( 1-\frac{\delta}{4} \right)^T \Delta_2(\family{}{}{x})
    \geq \left( 2^{- (N-1)/T_\brute{}(N)} \right)^{T_\brute{}(N)} \Delta_2(\family{}{}{x}) = 2 \frac{\Delta_2(\family{}{}{x})}{2^{N}} > \frac{\Delta_2(\family{}{}{x})}{2^{N}},
\end{equation}
which shows that the final vectors obtained by \brute{} violate asymptotic agreement.
\end{proof}

\section{\icwtm{} Algorithm}
In this section, we prove that \icwtm{} solves averaging agreement under Assumption \ref{ass:query}.

\subsection{Diameters}

Before doing so, we introduce the notion of $\ell_r$-diameters and we prove a few useful lemmas.

\begin{definition}
  For any $r \in [1,\infty]$, we define the diameter along coordinate $i$ by
  \begin{equation}
    \Delta^{cw}(\family{}{}{x})[i] = \max_{j,k \in [h]} \absv{\indexvar{j}{}{x}[i] - \indexvar{k}{}{x}[i]},
  \end{equation}
  and the coordinate-wise $\ell_r$-diameters by $\Delta^{cw}_r(\family{}{}{x}) = \norm{\Delta^{cw}(\family{}{}{x})}{r}$.
\end{definition}

Interestingly, we have the following bounds between diameters.

\begin{lemma}
\label{lemma:diameter}
The $\ell_r$-diameters are upper-bounded by coordinate-wise $\ell_r$-diameters, i.e.,
\begin{equation}
  \forall r \mathsep \Delta_r \leq \Delta^{cw}_r \leq \min \left\lbrace d^{1/r}, 2 h^{1/r} \right\rbrace \Delta_r.
\end{equation}
Note that in ML applications, we usually expect $d \gg h$, in which case the more relevant right-hand side inequality is $\Delta^{cw}_r \leq 2 h^{1/r} \Delta_r$.
\end{lemma}

\begin{proof}
Consider $j^*, k^* \in [h]$ such that $\Delta_r(\family{}{}{x}) = \norm{\indexvar{j^*}{}{x} - \indexvar{k^*}{}{x}}{r}$. But then, we note that on each coordinate $i \in [d]$,
\begin{equation}
  \absv{\indexvar{j^*}{}{x}[i] - \indexvar{k^*}{}{x}[i]} \leq \max_{j, k \in [h]} \absv{\indexvar{j}{}{x}[i] - \indexvar{k}{}{x}[i]} = \Delta^{cw}(x)[i].
\end{equation}
As a result, $\Delta_r(\family{}{}{x}) = \norm{\indexvar{j^*}{}{x} - \indexvar{k^*}{}{x}}{r} \leq \norm{\Delta^{cw}(\family{}{}{x})}{r} = \Delta^{cw}_r(\family{}{}{x})$. For the right-hand side, first note that a coordinate-wise diameter is smaller than the $\ell_r$ diameter, which yields
\begin{align}
    \Delta_r^{cw} (\family{}{}{x})^r
    &= \sum_{i \in [d]} \max_{j,k \in [h]} \absv{\indexvar{j}{}{x}[i] - \indexvar{k}{}{x}[i]}^r
    \leq \sum_{i \in [d]} \max_{j,k \in [h]} \norm{\indexvar{j}{}{x} - \indexvar{k}{}{x}}{r}^r \\
    &= \sum_{i \in [d]} \Delta_r(\family{}{}{x})^r
    = d \Delta_r (\family{}{}{x})^r.
\end{align}
Taking the $r$-th root shows that $\Delta^{cw}_r \leq d^{1/r} \Delta_r$. What is left to prove is that $\Delta^{cw}_r \leq 2h^{1/r} \Delta_r$. To prove this, note that for any $i \in [d]$, we have
\begin{equation}
\max_{j,k \in [h]} \absv{\indexvar{j}{}{x}[i] - \indexvar{k}{}{x}[i]} \leq \max_{j \in [h]} \left( 2 \absv{\indexvar{j}{}{x}[i] - \indexvar{1}{}{x}[i]} \right).
\end{equation}
Indeed, assuming the former maximum is reached for $j^*$ and $k^*$,  the latter maximum will be reached for $j^*$ or $k^*$, depending on whether $\indexvar{1}{}{x}[i]$ is closer to $\indexvar{j^*}{}{x}[i]$ or $\indexvar{k^*}{}{x}[i]$. In either case, the above inequality holds.
As a result,

\begin{align}
    \Delta_r^{cw} (\family{}{}{x})^r
    &= \sum_{i \in [d]} \max_{j,k \in [h]} \absv{\indexvar{j}{}{x}[i] - \indexvar{k}{}{x}[i]}^r
    \leq \sum_{i \in [d]} \max_{j \in [h]} \left( 2 \absv{\indexvar{j}{}{x}[i] - \indexvar{1}{}{x}[i]} \right)^r \\
    &= 2^r \sum_{i \in [d]} \max_{j \in [h]} \absv{\indexvar{j}{}{x}[i] - \indexvar{1}{}{x}[i]}^{r}
    \leq 2^r \sum_{i \in [d]} \sum_{j \in [h]} \absv{\indexvar{j}{}{x}[i] - \indexvar{1}{}{x}[i]}^{r} \\
    &= 2^r \sum_{j \in [h]} \sum_{i \in [d]} \absv{\indexvar{j}{}{x}[i] - \indexvar{1}{}{x}[i]}^{r}
    = 2^r \sum_{j \in [h]} \norm{\indexvar{j}{}{x} - \indexvar{1}{}{x}}{r}^r \\
    &\leq 2^r \sum_{j \in [h]} \Delta_r(\family{}{}{x})^r
    = 2^r h \Delta_r (\family{}{}{x})^r.
\end{align}
Taking the $r$-th root yields $\Delta^{cw}_r \leq 2h^{1/r} \Delta_r$, which concludes the proof.
\end{proof}

As an immediate corollary, asymptotic agreement
is equivalent to showing that \emph{any} of the diameters we introduce in this section goes to zero.

Interestingly, our diameters satisfy the triangle inequality, as shown by the following lemma.

\begin{lemma}
\label{lemma:inequality_cw_diameter}
The diameters and coordinate-wise diameters satisfy the triangle inequality. Namely, for any two families of vectors $\family{}{}{x}$ and $\family{}{}{y}$, we have the following inequality
\begin{equation}
  \Delta^{cw} (\family{}{}{x}+\family{}{}{y}) \leq \Delta^{cw}(\family{}{}{x}) + \Delta^{cw}(\family{}{}{y}).
\end{equation}
As an immediate corollary, by triangle inequality of norms, for any $r \in [1,\infty]$, we also have $\Delta^{cw}_r (\family{}{}{x}+\family{}{}{y}) \leq \Delta^{cw}_r(\family{}{}{x}) + \Delta^{cw}_r(\family{}{}{y})$. We also have $\Delta_r (\family{}{}{x}+\family{}{}{y}) \leq \Delta_r(\family{}{}{x}) + \Delta_r(\family{}{}{y})$.
\end{lemma}

\begin{proof}
For any coordinate $i \in [d]$, the following holds:
\begin{align}
  \Delta^{cw}(\family{}{}{x}+\family{}{}{y})[i] &= \max_{j,k \in [h]} \absv{\indexvar{j}{}{x}[i] + \indexvar{j}{}{y}[i] - \indexvar{k}{}{x}[i] - \indexvar{k}{}{y}[i]} \\
  &\leq \max_{j,k \in [h]} \left\lbrace \absv{\indexvar{j}{}{x}[i] - \indexvar{k}{}{x}[i]} +  \absv{\indexvar{j}{}{y}[i] - \indexvar{k}{}{y}[i]} \right\rbrace \\
  &\leq \max_{j,k \in [h]} \absv{\indexvar{j}{}{x}[i] - \indexvar{k}{}{x}[i]} + \max_{j',k' \in [h]} \absv{\indexvar{j'}{}{y}[i] - \indexvar{k'}{}{y}[i]} \\
  &= \Delta^{cw}(\family{}{}{x})[i] + \Delta^{cw}(\family{}{}{y})[i],
\end{align}
which concludes the proof for coordinate-wise diameters. The proof for $\ell_r$ diameters is similar.
\end{proof}

\subsection{Correctness proof of \icwtm{} and Lower bound on Byzantine tolerance}

We now move on to the proof of correctness of \icwtm{} for $n \geq 3f+1$. First, denoting $\indexvar{j}{}{S}[i] = S(\family{j}{}{z}[i])$, note that we have the following lemma.

\begin{lemma}
\label{lemma:cwtm_difference_sets}
   For any two honest nodes $j$ and $k$ and any coordinate $i$, we have
  \begin{equation}
    \card{\indexvar{j}{}{S}[i]-\indexvar{k}{}{S}[i]} \leq f.
  \end{equation}
\end{lemma}

\begin{proof}
Note that node $j$ receives messages from at most $n$ nodes, and in the trimming step, $2f$ nodes are discarded, which yields $\card{\indexvar{j}{}{S}[i]} \leq n - 2f$. Moreover, among the nodes in $\quorums{j}{} \cap \quorums{k}{}$ at most, $f$ nodes with the smallest and $f$ nodes with the largest $i$-th coordinates will be trimmed. Now recall that $\card{\quorums{j}{} \cap \quorums{k}{}} \geq q$, thus, $\card{\indexvar{j}{}{S}[i] \cap \indexvar{k}{}{S}[i]} \geq q-2f = n -3f$. We then obtain
  \begin{equation}
    \card{\indexvar{j}{}{S}[i]-\indexvar{k}{}{S}[i]} = \card{\indexvar{j}{}{S}[i]} - \card{\indexvar{j}{}{S}[i] \cap \indexvar{k}{}{S}[i]} \leq f,
  \end{equation}
  which is the lemma.
\end{proof}
We denote $\indexvar{min}{}{x}[i] = \min\limits_{j \in [h]} \indexvar{j}{}{x}[i]$ and $\indexvar{max}{}{x}[i] = \max\limits_{j \in [h]} \indexvar{j}{}{x}[i]$ the minimal and maximal $i$-th coordinate among the parameters of the honest nodes.

\begin{lemma}
\label{lemma:cwtm_range}
  All the values that are not discarded in the trimming step are within the range of the values proposed by the honest nodes, i.e.,
  \begin{equation}
    \forall i \in [d] \mathsep \forall j \in [h] \mathsep \forall k \in \indexvar{j}{}{S}[i] \mathsep \indexvar{min}{}{x}[i] \leq \indexvar{k}{}{w}[i] \leq \indexvar{max}{}{x}[i]
  \end{equation}
\end{lemma}
\begin{proof}
  Note that there exist at most $f$ Byzantine nodes that might broadcast vectors with the $i$-th coordinate larger than $ \indexvar{max}{}{x}[i]$ or smaller than $\indexvar{min}{}{x}[i]$, and all of these nodes will be removed by trimming.
\end{proof}

\begin{lemma}[Contraction by \cwtm]
\label{lemma:cwtm_contraction}
  Under Assumption \ref{ass:query}, \cwtm{} guarantees the contraction of the coordinate-wise diameters, that is,
  \begin{equation}
    \forall \family{}{}{x} \mathsep \forall \byzfamily{}{}, \Delta^{cw} \left( \cwtmfamily \circ \byzfamily{}{} \left( \family{}{}{x} \right) \right) \leq (1-\tilde \varepsilon)  \Delta^{cw} \left( \family{}{}{x} \right).
  \end{equation}
  As an immediate corollary, this inequality holds by taking the $\ell_r$-norm on both sides, which means that the coordinate-wise $\ell_r$-diameter is also contracted by the same factor.
\end{lemma}

\begin{proof}
  Let us bound the distance between $\indexvar{j}{}{y}[i]$ and $\indexvar{k}{}{y}[i]$, the $i$-th coordinate of the outputs of two  arbitrary nodes $j$ and $k$ after a Byzantine attack. Denote
  \begin{equation}
     m=\frac{1}{\card{\indexvar{j}{}{S}[i]\cap \indexvar{k}{}{S}[i]}} \sum_{l\in \indexvar{j}{}{S}[i]\cap\indexvar{k}{}{S}[i]} \indexvar{l}{}{w}[i],
  \end{equation}
   the average of the $i$-th coordinates of the nodes in $\indexvar{j}{}{S}[i] \cap \indexvar{k}{}{S}[i]$. Without loss of generality, assume $\indexvar{j}{}{y}[i] \geq \indexvar{k}{}{y}[i]$. We then obtain

  \begin{align}
    &\card{\indexvar{j}{}{y}[i]-\indexvar{k}{}{y}[i]}
    = {\frac{1}{\card{\indexvar{j}{}{S}[i]}} \sum_{l\in \indexvar{j}{}{S}[i]} \indexvar{l}{}{w}[i] - \frac{1}{\card{\indexvar{k}{}{S}[i]}} \sum_{l\in \indexvar{k}{}{S}[i]} \indexvar{l}{}{w}[i]} \\
    &= \left( m + \frac{1}{\card{\indexvar{j}{}{S}[i]}} \sum_{l \in \indexvar{j}{}{S}[i]-(\indexvar{j}{}{S}[i]\cap \indexvar{k}{}{S}[i])} (\indexvar{l}{}{w}[i]-m)  \right) \\
    &- \left( m + \frac{1}{\card{\indexvar{k}{}{S}[i]}} \sum_{l\in \indexvar{k}{}{S}[i]-(\indexvar{j}{}{S}[i]\cap \indexvar{k}{}{S}[i])} (\indexvar{l}{}{w}[i]-m)  \right) \\
    &=\frac{1}{\card{\indexvar{j}{}{S}[i]}} \sum_{l\in \indexvar{j}{}{S}[i]-\indexvar{k}{}{S}[i]} (\indexvar{l}{}{w}[i]-m) - \frac{1}{\card{\indexvar{k}{}{S}[i]}} \sum_{l\in \indexvar{k}{}{S}[i]-\indexvar{j}{}{S}[i]} (\indexvar{l}{}{w}[i]-m) \\
    &\leq \frac{1}{\card{\indexvar{j}{}{S}[i]}} \sum_{l\in \indexvar{j}{}{S}[i]-\indexvar{k}{}{S}[i]} (\indexvar{max}{}{x}[i]-m) - \frac{1}{\card{\indexvar{k}{}{S}[i]}} \sum_{l\in \indexvar{k}{}{S}[i]-\indexvar{j}{}{S}[i]} (\indexvar{min}{}{x}[i]-m) \\
    &= \frac{\card{\indexvar{j}{}{S}[i]-\indexvar{k}{}{S}[i]}}{\card{\indexvar{j}{}{S}[i]}}(\indexvar{max}{}{x}[i]-m) + \frac{\card{\indexvar{k}{}{S}[i]-\indexvar{j}{}{S}[i]}}{\card{\indexvar{k}{}{S}[i]}}(m-\indexvar{min}{}{x}[i]),\label{eq: contaction}
  \end{align}
  where the inequality uses Lemma \ref{lemma:cwtm_range}. Note that Lemma \ref{lemma:cwtm_range} also implies that $\indexvar{min}{}{x}[i] \leq m \leq \indexvar{max}{}{x}[i]$ since $m$ is the average of some real numbers, all of which are within the range of the values proposed by the honest nodes. Now notice that using Lemma \ref{lemma:cwtm_difference_sets} we have
  \begin{align}
    \frac{\card{\indexvar{j}{}{S}[i]-\indexvar{k}{}{S}[i]}}{\card{\indexvar{j}{}{S}[i]}} &= \frac{\card{\indexvar{j}{}{S}[i]-\indexvar{k}{}{S}[i]}}{\card{\indexvar{j}{}{S}[i]-\indexvar{k}{}{S}[i]}+\card{\indexvar{j}{}{S}[i]\cap \indexvar{k}{}{S}[i]}} \\
    &\leq \frac{f}{f+\card{\indexvar{j}{}{S}[i]\cap \indexvar{k}{}{S}[i]}} \\
    &\leq \frac{f}{f+\varepsilon f} =  1-\tilde \varepsilon,
  \end{align}
  where we used the fact that $\card{\indexvar{j}{}{S}[i]\cap \indexvar{k}{}{S}[i]}\geq q-2f = n - 3f \geq \varepsilon f$, and similarly,
  \begin{equation}
    \frac{\card{\indexvar{k}{}{S}[i]-\indexvar{j}{}{S}[i]}}{\card{\indexvar{k}{}{S}[i]}} \leq 1- \tilde \varepsilon.
  \end{equation}
  Combining these inequalities with Equation (\ref{eq: contaction}), we then obtain
  \begin{equation}
    \card{\indexvar{j}{}{y}[i]-\indexvar{k}{}{y}[i]} \leq \left( 1- \tilde \varepsilon \right) \left(\indexvar{max}{}{x}[i]-\indexvar{min}{}{x}[i]\right)
    = \left( 1- \tilde \varepsilon \right) \Delta^{cw} \left( \family{}{}{x} \right)[i].
  \end{equation}
  Therefore, we have
  \begin{equation}
    \Delta^{cw} \left( \family{}{}{y} \right)[i] \leq \left( 1- \tilde \varepsilon \right) \Delta^{cw} \left( \family{}{}{x} \right)[i],
  \end{equation}
  which is what we wanted.
\end{proof}

\begin{remark}
Note that, in the regime $f \ll n$, we can set $\varepsilon \rightarrow \infty$, in which case we have $\tilde \varepsilon \rightarrow 1$.
Thus, for any fixed value of $\aggrparameter$, when $f$ is sufficiently small compared to $n$, \icwtm{} actually achieves asymptotic agreement in only one communication round.
\end{remark}

We have the following corollary regarding \icwtm{}.

\begin{corollary}
\label{lemma:cwtm_agreement}
Under Assumption \ref{ass:query}, for any input $\aggrparameter \in \mathbb N$, the algorithm \icwtm{} guarantees Byzantine asymptotic agreement, i.e.,
\begin{equation}
    \Delta_2 \left( \icwtmfamily{}{}_\aggrparameter \circ \byzfamily{}{\aggrparameter} (\family{}{}{x}) \right) \leq \frac{\Delta_2 \left( \family{}{}{x} \right)}{2^\aggrparameter}.
\end{equation}
\end{corollary}

\begin{proof}
  First note that since \icwtm{} iterates \cwtm{}, there is actually a sequence of attacks $\byzfamily{}{t}$ at each iteration $t \in \left[ T_{\icwtm{}} (\aggrparameter) \right]$. In fact, we have a sequence of families $\family{}{t}{y}$ defined by $\family{}{0}{y} \triangleq \family{}{}{x}$ and $\family{}{t+1}{y} \triangleq \cwtmfamily{} \circ \byzfamily{}{t} \left( \family{}{t}{y} \right)$ for $t \in \left[ T_{\icwtm{}}(\aggrparameter) -1 \right]$.
  We eventually have $\icwtmfamily{}{} \circ \byzfamily{}{} (\family{}{}{x}) = \family{}{T_{\icwtm{}}(\aggrparameter)}{y}$.

  Note that the previous lemma implies that
  \begin{equation}
      \Delta^{cw} \left(\family{}{t+1}{y} \right) \leq \left( 1- \tilde \varepsilon \right) \Delta^{cw} \left( \family{}{t}{y} \right).
  \end{equation}
  Taking the $\ell_2$ norm on both sides then implies that
  \begin{equation}
      \Delta^{cw}_2 \left(\family{}{t+1}{y} \right) = \normtwo{\Delta^{cw} \left(\family{}{t+1}{y} \right)} \leq \left( 1- \tilde \varepsilon \right) \normtwo{\Delta^{cw} \left( \family{}{t}{y} \right)} = \left( 1- \tilde \varepsilon \right) \Delta^{cw}_2 \left( \family{}{t}{y} \right).
  \end{equation}
  It follows straightforwardly that
  \begin{align}
      \Delta_2^{cw} \left( \family{}{T_{\icwtm{}} (\aggrparameter)}{y} \right)
      &\leq \left( 1- \tilde \varepsilon \right)^{T_{\icwtm{}} (\aggrparameter)} \Delta_2^{cw} \left( \family{}{}{x} \right) \\
      &\leq \left( 1- \tilde \varepsilon \right)^{ \frac{(\aggrparameter+1) \ln 2 + \ln \sqrt{h}}{ \tilde \varepsilon }} \Delta_2^{cw} \left( \family{}{}{x} \right)  \\
      &= \exp \left( \frac{\ln(1-\tilde \varepsilon) }{\tilde \varepsilon} ((\aggrparameter+1) \ln 2 + \ln \sqrt{h}) \right) \Delta_2^{cw} \left( \family{}{}{x} \right)  \\
      &\leq \exp \left( - (\aggrparameter+1) \ln 2 \right) \exp \left( - {\ln \sqrt h} \right) \Delta_2^{cw} \left( \family{}{}{x} \right)\label{eqn:195} \\
      &= \frac{1}{2^{1+\aggrparameter}\sqrt{h}} \Delta_2^{cw} \left( \family{}{}{x} \right),
  \end{align}
  where, in Equation (\ref{eqn:195}), we used $\ln(1+u) \leq u$ for $u \in (-1,0]$. We now conclude by invoking Lemma \ref{lemma:diameter}, which implies
  \begin{equation}
      \Delta_2(\family{}{T_{\icwtm{}} (\aggrparameter)}{y}) \leq \Delta^{cw}_2(\family{}{T_{\icwtm{}} (\aggrparameter)}{y}) \leq 2^{-\aggrparameter} \frac{\Delta_2^{cw}(\family{}{}{x})}{2\sqrt{h}} \leq 2^{-\aggrparameter} \Delta_2 \left( \family{}{}{x} \right),
  \end{equation}
  which proves that \icwtm{} achieves asymptotic agreement.
\end{proof}

We now prove our theorem.

\begin{proof}[Proof of theorem \ref{th:cwtm}]
  Consider a family $\family{}{0}{x} \in \mathbb R^{d \cdot h}$. We first focus on coordinate $i \in [d]$ only. We sort the family using a permutation $\sigma$ of $[h]$, so that
  \begin{equation}
      \indexvar{\sigma(1)}{0}{x}[i] \leq \indexvar{\sigma(2)}{0}{x}[i] \leq \ldots \leq \indexvar{\sigma(h-1)}{0}{x}[i] \leq \indexvar{\sigma(h)}{0}{x}[i].
  \end{equation}
  Now denote $q_j=\card{\quorums{j}{}}$, and $\family{}{}{z} = \byzfamily{j}{0} (\family{}{0}{x}) = \family{Q^{(j)}}{0}{w} \in \mathbb R^{d \cdot q_j}$ the result of a Byzantine attack. Again, we sort the vectors of this family, using a permutation $\tau$ of $[q_j]$, so that
  \begin{equation}
      \indexvar{\tau(1)}{}{z}[i] \leq \indexvar{\tau(2)}{}{z}[i] \leq \ldots \leq \indexvar{\tau(q_j-1)}{}{z}[i] \leq \indexvar{\tau(q_j)}{}{z}[i].
  \end{equation}
  Now, denoting $y \triangleq \cwtm{} (\family{}{}{z})$, we note that
  \begin{equation}
      y[i] = \frac{1}{q_j-2f} \sum_{k=1}^{q_j-2f} \indexvar{\tau(f+k)}{}{z}[i].
  \end{equation}
  Moreover, note that there are $f+k-1$ values of $\family{}{}{z}$ that are smaller than $\indexvar{\tau(f+k)}{}{z}[i]$. These can include $f$ Byzantine vectors. But the remaining $k-1$ values must then come from the family of honest vectors. Yet, the $k-1$ smallest vectors of this family are $\indexvar{\sigma(1)}{0}{x}[i], \ldots, \indexvar{\sigma(k-1)}{0}{x}[i]$. But then, $\indexvar{\tau(f+k)}{}{z}[i]$ will have to take a value on the right of $\indexvar{\sigma(k-1)}{0}{x}[i]$ in the list of honest vectors, which corresponds to saying that
  \begin{equation}
      \forall k \in [q_j-f] \mathsep \indexvar{\tau(f+k)}{}{z}[i] \geq \indexvar{\sigma(k)}{0}{x}[i].
  \end{equation}
  But then, we know that
  \begin{equation}
      y[i] \geq \frac{1}{q_j-2f} \sum_{k=1}^{q_j-2f} \indexvar{\sigma(k)}{0}{x}.
  \end{equation}
  As an immediate corollary, we see that $y[i] \geq \indexvar{\sigma(1)}{0}{x}[i]$, which also implies that
  \begin{equation}
      \indexvar{\sigma(k)}{0}{x}[i] \leq \indexvar{\sigma(1)}{0}{x}[i] + \max_{l \in [h]} \left( \indexvar{\sigma(l)}{0}{x}[i] - \indexvar{\sigma(1)}{0}{x}[i] \right) \leq y[i] + \Delta^{cw}(\family{}{0}{x})[i].
  \end{equation}
  But now notice that
  \begin{align}
      \indexvar{}{0}{\bar x}[i]
      &= \frac{1}{h} \sum_{k=1}^h \indexvar{\sigma(k)}{0}{x} = \frac{1}{h} \sum_{k=1}^{q_j-2f} \indexvar{\sigma(k)}{0}{x} + \frac{1}{h} \sum_{k=q_j-2f+1}^{h} \indexvar{\sigma(k)}{0}{x} \\
      &\leq \frac{1}{h} \left( (q_j-2f) y[i] \right) + \frac{1}{h} \sum_{k=q_j-2f+1}^{h} \left(  y[i] + \Delta^{cw}(\family{}{0}{x})[i] \right) \\
      &= y[i] + \frac{h-q_j+2f}{h} \Delta^{cw}(\family{}{0}{x})[i].
  \end{align}
  Similarly, we can also prove that $\indexvar{}{0}{\bar x}[i] \geq y[i] - \frac{h-q_j+2f}{h} \Delta^{cw}(\family{}{0}{x})[i]$, which implies that
  \begin{equation}
      \absv{y[i] - \indexvar{}{0}{\bar x}[i]} \leq \frac{h-q_j+2f}{h} \Delta^{cw}(\family{}{0}{x})[i] \leq \frac{2f}{h}\Delta^{cw}(\family{}{0}{x})[i],
  \end{equation}
  where we used the fact that $q_j\geq q = h$. Thus, $\normtwo{y-\indexvar{}{0}{\bar x}} \leq \frac{2f}{h} \normtwo{\Delta^{cw}(\family{}{0}{x})} = \frac{2f}{h} \Delta^{cw}_2(\family{}{0}{x})$. In fact, more generally, we showed that, for any Byzantine attack $\byzfamily{j}{0}$, we have
  \begin{equation}
      \cwtm{} \circ \byzfamily{j}{0} (\family{}{0}{x}) \in Y_0 = \indexvar{}{0}{\bar x} + \frac{2f}{h} \prod_{i \in [d]} \left[ - \Delta^{cw}(\family{}{0}{x})[i], +\Delta^{cw}(\family{}{0}{x})[i] \right].
  \end{equation}
  Yet Lemma \ref{lemma:cwtm_range} shows that any such parallelepiped was stable under application of \cwtm{} despite Byzantine attacks. Thus, for any iteration $t \geq 1$, we still have $\indexvar{j}{t}{x} \in Y_0$, which then guarantees that
  \begin{align}
      \normtwo{\bar x_t - \bar x_0}
      &= \normtwo{\frac{1}{h} \sum_{j \in [h]} \left( \indexvar{j}{t}{x} - \bar x_0 \right)}
      \leq \frac{1}{h} \sum_{j \in [h]} \normtwo{ \indexvar{j}{t}{x} - \bar x_0 } \\
      &\leq \frac{1}{h} \sum_{j \in [h]} \frac{2f}{h} \normtwo{\Delta^{cw}(\family{}{0}{x})} = \frac{2f}{h} \Delta_2^{cw}(\family{}{0}{x}).
  \end{align}
  Lemma \ref{lemma:diameter} then guarantees $\Delta_2^{cw}(\family{}{0}{x}) \leq 2\sqrt{h} \Delta_2(\family{}{0}{x})$. By noting that \icwtm{} corresponds to iterating \cwtm{}, we conclude that \icwtm{} achieves $\frac{4f}{\sqrt{h}}$-averaging agreement.

Now we show that for $n \leq 3f$, no algorithm can achieve Byzantine averaging agreement. If $n \leq 3f$, then $h = n-f \leq 2f$. Thus honest nodes can be partitioned into two subsets of cardinals at most $f$. In particular, for any subset, Byzantine nodes can block all messages coming from the other subset. Any subset would thus only hear from nodes of the subset and from the Byzantine nodes.

Assume now by contradiction that \aggr{} achieves Byzantine-resilience averaging agreement for $n \leq 3f$. Note that we then have $q \leq 2f$. As a result $q-f \leq f$.
But as a result, if Byzantines send $\family{}{}{z} = z \star f$ to all honest nodes, quasi-unanimity (Lemma \ref{lemma:quasi-unanimity}) applies, which means that all honest nodes must output $z$.

But this hold for any value $z$ chosen by the Byzantine nodes. Clearly, this prevents averaging. Thus \aggr{} fails to achieve averaging agreement.
\end{proof}

\end{document}